\documentclass[twoside]{article}
\usepackage[accepted]{aistats2023}

\usepackage{bm}
\usepackage{aistats2023}
\usepackage{url}
\usepackage{amsthm}
\usepackage{mathrsfs}
\usepackage{cite}
\usepackage{mathtools}
\usepackage{amssymb}
\usepackage{latexsym}
\usepackage{mathtext}
\usepackage{mathspic}
\usepackage{color}
\usepackage{amssymb}
\usepackage{color}
\usepackage{amsfonts}

\newtheorem{definition}{Definition}[section]

\newtheorem{lemma}{Lemma}[section]

\newtheorem{assumption}{Assumption}[section]
\usepackage{graphicx}
\usepackage{multirow}
\usepackage{multicol}
\usepackage{subcaption}
\usepackage{booktabs} 
\usepackage{babel}
\usepackage{threeparttable}
\usepackage[ruled, linesnumbered]{algorithm2e}

\newcommand{\E}{\mathbb{E}}
\newcommand{\Probability}{\mathbb{P}}
\newcommand{\Regsq}{\text{Reg}_\text{sq}}
\newcommand{\RegCB}{\text{Reg}_\text{CB}}

\newtheorem{theorem}{Theorem}[section]

\newcommand{\OPTDP}{\text{OPT}_{\text{DP}}}

\newcommand{\mD}{\mathcal{D}}

\usepackage{enumitem}

\newcommand{\mF}{\mathcal{F}}

\newcommand{\mT}{\mathcal{T}}

\newcommand{\OPT}{\text{OPT}}
\newcommand{\Reg}{{\text{Reg}}}

\newcommand{\mG}{\mathcal{G}}

\newcommand{\mX}{{\mathcal X}}

\usepackage[round]{natbib}

\begin{document}

%

%

\twocolumn[

\aistatstitle{Optimal Contextual Bandits with Knapsacks under Realizability via Regression Oracles
}

\aistatsauthor{ Yuxuan Han$^{*\dagger}$ \And Jialin Zeng$^{*\dagger}$ \And  Yang Wang$^{\dagger \ddagger }$ \And Yang Xiang$^{\dagger \S}$  \And Jiheng Zhang$^{\ddagger}$ }

\aistatsaddress{$\dagger$ Department of Mathematics, HKUST \\   $\ddagger$Department of Industrial Engineering and Decision Analytics, HKUST \\  $\S$HKUST Shenzhen-Hong Kong Collaborative Innovation Research Institute }]
\begin{abstract}
We study the stochastic contextual bandit with knapsacks (CBwK) problem, where each action, taken upon a context, not only leads to a random reward but also costs a random resource consumption in a vector form. 
The challenge is to maximize the total reward without violating the budget for each resource.
We study this problem under a general realizability setting where the expected reward and expected cost are functions of contexts and actions in some given general function classes $\mathcal{F}$ and $\mathcal{G}$, respectively.
Existing works on CBwK are restricted to the linear function class since they use UCB-type algorithms, which heavily rely on the linear form and thus are difficult to extend to general function classes.
Motivated by online regression oracles that have been successfully applied to contextual bandits, we propose the first universal and optimal algorithmic framework for CBwK by reducing it to online regression. 
We also establish the lower regret bound to show the optimality of our algorithm for a variety of function classes.
\end{abstract}

\allowdisplaybreaks
\hyphenpenalty=5000
\tolerance=1000

\section{INTRODUCTION}

Contextual Bandits (CB) is a fundamental online learning framework with an exploration-exploitation tradeoff. 
At each step, the decision maker, to maximize the total reward, takes one out of $K$ actions upon observing the context and then receives a random reward.
It has received extensive attention due to its wide range of applications, such as recommendation systems, clinic trials, and online advertisement \citep{bietti2021contextual,slivkins2019introduction,lattimore2020bandit}.

However, the standard contextual bandit setting ignores the budget constraint that commonly arises in real-world applications. 
For example, when a search engine picks an ad to display, it needs to consider both the advertiser's feature and his remaining budget.
Thus, canonical contextual bandits have been generalized to Contextual Bandit with Knapsacks (CBwK) setting, where a selected action at each time $t$ will lead to a context-dependent consumption of several constrained resources and there exists a global budget $B$ on the resources.

In contextual decision-making, it is crucial to model the relationship between the outcome (reward and costs) and the context to design an effective and efficient policy.
In a non-budget setting, i.e., without the knapsack constraints, the contextual bandit algorithms fall into two main categories: agnostic and realizability-based approaches.
The agnostic approaches aim to find the best policy out of a given policy set $\Pi$ and make no assumptions on the context-outcome relationships. 
In this case, one often requires access to a cost-sensitive classification oracle over $\Pi$ to achieve computational efficiency \citep{langford2007epoch,dudik2011efficient,agarwal2014taming}. 
On the other hand, realizability-based approaches assume that there exists a given function class $\mathcal{F}$ that models the outcome distribution with the given context. 
When $\mathcal{F}$ is a linear class, the optimal regret has been achieved by UCB-type algorithms \citep{chu2011contextual,abbasi2011improved}.
Moreover, a unified approach to developing optimal algorithms has recently been studied in \citet{foster2018practical}, \cite{foster2020beyond}, \cite{simchi2021bypassing} by introducing regression oracles over the class $\mathcal F$. 
While the realizability assumption could face the problem of model mismatch, it has been shown that empirically realizability-based approaches outperform agnostic-based ones when no model misspecification exists \citep{krishnamurthy2016contextual,foster2018practical}. 
However, these oracle-based algorithms only apply to unconstrained contextual bandits without considering the knapsack constraints.

In the CBwK setting, two similar approaches as above have been explored.
For the agnostic approach, \cite{badanidiyuru2014resourceful} and \cite{agrawal2016efficient}  extend the general contextual bandits to the knapsack setting by generalizing the techniques in contextual bandits without knapsack constraints \citep{dudik2011efficient, agarwal2014taming}.
However, as shown in \cite{agrawal2016linear}, the optimization problem using classification oracle in the agnostic setting could be NP-hard even for the linear case. 
On the other hand, CBwK under linear realizability assumption (i.e., both the reward function class $\mathcal{F}$ and the cost function class $\mathcal{G}$ are linear function classes) and its variants have been studied.
Based on the confidence ellipsoid results guaranteed by the linear model, \cite{agrawal2016linear} develop algorithms for the stochastic linear CBwK with near-optimal regret bounds.
\cite{sivakumar2022smoothed} study the linear CBwK under the smoothed contextual setting where the contexts are perturbed by Gaussian noise.
However, to our knowledge, no work has considered realizability beyond the linear case.
The difficulty lies in the lack of confidence ellipsoid results for the general function class as in linear model assumptions.
In this paper, we try to address the following open problem:

\textit{Is there a unified algorithm that applies to CBwK under the general realizability assumption?}

We answer this question positively by proposing \textbf{SquareCBwK}, the first general and optimal algorithmic framework for stochastic CBwK based on online regression oracles. 
Compared to the contextual bandits, additional resource constraints in CBwK lead to more intricate couplings across contexts, which makes it substantially more challenging to strike a balance between exploration and exploitation. 
It is thus nontrivial to apply regression oracle techniques to CBwK.
The key challenge lies in designing a proper score of actions for decision-making that can well balance the reward and the cost gained over time to maximize the rewards while ensuring no resource is run out.

\subsection{Our contributions}
In this paper, we address the above issue and successfully apply regression oracles to CBwK under the realizability assumption. Our contributions are summarized in the following three aspects.

\noindent\textbf{Algorithm for CBwK with Online Regression Oracles:} 
In Section~\ref{sec-squarecbwk}, we propose the first unified algorithm SquareCBwK for solving CBwK under the general realizability assumption, which addresses the above open problem.
Motivated by the works in contextual bandits that apply regression oracles to estimate the reward (to which we refer as the reward oracle), we propose to apply a new cost oracle in SquareCBwK to tackle the additional resource constraints in CBwK. 
Specifically, under the setting that the budget $B$ scales linearly with the time horizon $T$,  we construct a penalized version of the unconstrained score based on the cost oracle and adjust the penalization adaptive to the resource consumption over time.
Our algorithm is able to be instantiated with any available efficient oracle for any general function class to obtain an upper 
regret bound for CBwK. 
In this sense, we provide a meta-algorithm that reduces CBwK to regression problems, which is a well-studied area in machine learning with various computationally efficient algorithms for different function classes. In Section~\ref{sec-example}, we show that through different instantiations, SquareCBwK can achieve the optimal regret bound in the linear CBwK and derive new guarantees for more general function classes like non-parametric classes.

\noindent \textbf{Lower Bound Results:} In Section~\ref{sec-optimality}, we develop a new approach for establishing lower bounds for CBwK under the general realizability assumption. Previous works in CBwK under the realizability assumption usually show the optimality of their results by matching the lower bound in the unconstrained setting. However, when the cost function class $\mathcal G$ is more complicated than reward function class $\mathcal F$, the lower bound in the unconstrained setting will be very loose since it does not consider the information of $\mathcal G$. Our method, in contrast, can provide lower bounds that include information from both $\mathcal F$ and $\mathcal G$. 
Moreover, we apply this new method to construct lower bounds that match the regret bound given by SquareCBwK for various choices of $\mathcal{F}$ and $\mathcal{G}$, which demonstrates SquareCBwK's optimality for different function classes.

\noindent\textbf{Relaxed Assumption on the Budget: }
While we present our main result under the regime that the budget $B$ scales linearly with the time horizon $T$, such an assumption may be restrictive compared with previous works.
For example, in the linear CBwK, both \cite{sivakumar2022smoothed} and \cite{agrawal2016linear} allow a relaxed budget $B = \Omega(T^{3/4})$.
To close this gap, in Section~\ref{sec-general-B}, we design a two-stage algorithm based on SquareCBwK that allows a weaker condition on the budget for solving CBwK under the general realizability assumption.
In particular, in the linear CBwK setting, our relaxed condition recovers the $B = \Omega(T^{3/4})$ condition.

\subsection{Other related works}
\noindent\textbf{Bandit with Knapsack and Demand Learning:}
 One area closely related to our work is bandits with knapsacks (BwK), which does not consider contextual information. The non-contextual BwK problem is first investigated in a general formulation by \cite{badanidiyuru2018bandits} and further generalized to concave reward/convex constraint setting by \cite{agrawal2014bandits}.
Both papers use the idea of Upper-Confidence Bound as in the non-constrained setting \citep{auer2002finite}.
\cite{ferreira2018online} apply the BwK framework to the network revenue management problem and develop a Thompson-sampling counterpart for BwK.
Their assumption on the demand function is further generalized in the recent works \citep{miao2021network,chen2022fairness}.

\noindent\textbf{Online Optimization with Knapsacks:} 
One closely related area is online optimization with knapsacks, which can be seen as a full-information variant of the BwK problem: after a decision is made, the feedback for all actions is available to the learner.
Such problem often leads to solving online linear/convex programming, which has been studied in \cite{balseiro2022best}, \cite{jenatton2016adaptive}, \cite{agrawal2014fast}, \cite{mahdavi2012trading}, \cite{liu2022online} 
 and \cite{castiglioni2022online}.
In the work of \cite{liu2022online}, they study online contextual decision-making with knapsack constraints. Specifically, they study the continuous action setting by developing an online primal-dual algorithm based on the ``Smart Predict-then-Optimize'' framework \citep{elmachtoub2022smart} in the unconstrained setting and leave the problem with bandit feedback open.
Our work provides a partial answer to this open problem in the finite-armed setting.

\noindent\textbf{Concurrent works in CBwK:}
During the review period of our paper, two very recent works appeared on the arxiv.org \citep{slivkins2022efficient,ai2022re} which also consider the CBwK problem with general function classes. 

In independent and concurrent work, \cite{slivkins2022efficient} consider the CBwK problem via regression oracles similar to our work. They propose an algorithm based on the perspective of Lagrangian game that results in a slightly different choice of arm selection strategy than ours. They also attain the same regret bound up to scalar as our Theorem~\ref{thm-regret-upper-bound}. Nevertheless, the focus of their work is limited to the setting described in section~\ref{sec-online-regret} (i.e., the $B = \Omega(T)$ regime), where they provide only an upper regret bound. In contrast, our work demonstrates the optimality by proving lower bound results for $B = \Omega(T)$ regime. Moreover, we derive a two-stage algorithm and corresponding regret guarantees for the $B = o(T)$ regime that requires less stringent budget assumptions.

Another recent work \citep{ai2022re} study the CBwK problem with general function classes by combining the re-solving heuristic and the distribution estimation techniques. Their result is interesting in that it may work in some function classes without an efficient online oracle. They also achieve logarithmic regret under suitable conditions thus are more problem-dependent, in contrast to the oracle-based approach, as discussed in section~4 of \cite{foster2020beyond}. However, their framework involves a distribution estimation procedure that requires additional assumptions about the regularity of the underlying distribution. Consequently, their method's regret is influenced by the regularity parameters, potentially leading to suboptimal results.

\section{PRELIMINARIES}
\paragraph{Notations} Throughout this paper, $a\lesssim b \text{ and } a = O(b)$ $ \big (a \gtrsim b\text{ and }a = \Omega(b)\big )$ means $a\leq C b$ $ \big (a \geq Cb\big )$  for some absolute constant $C$. $a = \tilde{O}(b)$   $\big (a = \tilde{\Omega}(b)\big ) $ means  $a = O(b\max\{1,\text{polylog}(b)\} ) $ $ \big(a = \Omega(b\max\{1,\text{polylog}(b)\})\big )\\$ and $a\asymp b$ means $a = O(b)$ and $b = O(a)$.

\subsection{Basic setup}
We consider the stochastic CBwK setting.
Given the budget $\bm B \in \mathbb{R}^d$ for $d$ different resources and the time horizon $T$,
at each step, the decision maker needs to select an arm $a_t$ $\in [K]$ upon observing a context $x_t\in \mathcal{X}$ drawn i.i.d. from some unknown distribution $P_{\mathcal{X}}$.
Then a reward $r_{t,a_t} \in [0,1]$ and a consumption vector $\bm c_{t,a_t}\in [0,1]^d$ is observed. We assume that $r_{t,a}$ and $\bm c_{t,a}$ are generated i.i.d. from a fixed distribution parameterized by the given $x_t$ and $a$.
The goal is to learn a policy, a mapping from context to action, that maximizes the total reward while ensuring that the consumption of each resource does not exceed the budget. 

Without loss of generality we assume $B_1= B_2=\dots = B_d=B$.
We focus on the regime $B=\Omega(T)$.
That is the budget scales linearly with time.
(We relax this assumption on the budget in Section \ref{sec-general-B}.)
Moreover, we make a standard assumption that the $K$-th arm is a null arm that generates no reward or consumption of any resource when being pulled.

Similar to the unconstrained setting \citep{foster2018practical,foster2020beyond,simchi2021bypassing}, we assume access to two classes of functions $\mathcal{F}\subset (\mathcal{X}\times [K]\rightarrow{[0,1]})$ and $ \mathcal{G} \subset (\mathcal{X}\times [K]\rightarrow{[0,1]^d}) $ that characterize the expectation of reward and consumption distributions respectively.
Note that only one function class $\mathcal{F}$ for the reward distribution is considered in the unconstrained setting, while to fit CBwK, we add another regression class $\mathcal{G}$ to model the consumption distribution.
We assume the following realizability condition.

\begin{assumption}\label{assumption_realizable}
  There exists some
  $f^{\star}\in \mathcal{F}$ and $\bm g^{*}\in \mathcal{G}$  such that $f^*(x,a)=\mathbb{E}[r_{t,a}|x_t = x]$ and $\bm g^*(x,a)=\mathbb{E}[\bm c_{t,a}| x_t = x ], \forall a\in [K]$.
\end{assumption}

The algorithm benchmark is the best dynamic policy, which knows the contextual distribution $P_{\mathcal{X}}$, $f^{*}$, and $\bm g^{*}$ and can dynamically maximize total reward given the historical information and the current context.
We denote $\OPTDP$ as the expected total reward of the best dynamic policy.
The goal of the decision-maker is to minimize the following regret:
\begin{definition}\label{def-regret}
    $\Reg(T):= \OPTDP- \E[\sum_{t=1}^{\tau} r_{t,a_t}]$, where $\tau$ is the stopping time when there exists some $j\in [d]$ s.t. $\sum_{t=1}^{\tau} (\bm c_{t,a_t})_j>B-1$.
\end{definition}

Due to the intractability of the best dynamic policy in Definition~\ref{def-regret}, we consider a static relaxation problem that provides an upper bound of $\OPTDP$.

Denote $\Delta^K$ the set of probability distributions over the action set $[K]$. The following program aims to find the best static randomized policy $p^*:\mathcal{X}\rightarrow \Delta^K$ that maximizes the expected per-round reward while ensuring resources are not exceeded in expectation.
\begin{equation}
\begin{aligned}\label{eq-OPT-static}
	&\max_{p: \mathcal{X} \to \Delta^K} \E_{x\sim P_{\mathcal{X}} }[\sum_{a\in [K]} p_a(x) f^*(x,a)]\\
 &\text{ s.t. } \E_{x\sim P_{\mathcal{X}}}[\sum_{a\in [K]} p_a(x) \bm{g}^{*}(x,a)] \leq B/T\cdot\bm 1.
\end{aligned}
\end{equation}

\begin{lemma}\label{lem-static-bound}
  Let $\OPT$ denote the value of the optimal static policy \eqref{eq-OPT-static}, then we have $ T \OPT \geq   \OPTDP $.  
\end{lemma}
Lemma~\ref{lem-static-bound} can be derived directly following the proof of Lemma~1 in \cite{agrawal2016linear}, where they consider the linear CBwK case but the reasoning therein is completely independent of the linear structure. With Lemma~\ref{lem-static-bound}, we can control the regret bound by considering $T\OPT - \E[\sum_{t = 1}^\tau r_{t,a_t}].$
 

\subsection{Online Regression Oracles}

Under the realizability Assumption~\ref{assumption_realizable}, We introduce the online regression problem and the notion of online regression oracles.

The general setting of online regression problem with input space $\mathcal{Z}$, output space $\mathcal{Y}$, function class $\mathcal{H}$ and loss $\ell: \mathcal{Y}\times \mathcal{Y} \to \mathbb{R}_+$ is described as following: At the beginning of the game, the environment chooses an $h^*: \mathcal{Z} \to \mathcal{Y} , h^*\in \mathcal{H}$ as the underlying response generating function. Then at each round $t$, (i) the learner receives an input $z_t \in \mathcal{Z}$ possibly chosen in adversarial by the environment, (ii) the learner then predicts a value $\hat{y}(z_t)\in \mathcal{Y}$ based on historical information, (iii) the learner observes a noisy response of $h^*(z_t)$ and suffers a loss $\ell(\big(\hat{y}(z_t), h^*(z_t)\big)$. The goal of the learner is to minimize the cumulative loss \begin{align*}
 \Regsq(T;h^*): = \sum_{t = 1}^T \ell\big(\hat{y}(\bm z_t), h^*(z_t)\big).
\end{align*}
In our problem, we assume our access to oracles $\mathcal{R}^r,\mathcal{R}^c$ for two online regression problems, respectively:

\noindent \textbf{Reward Regression Oracle} The reward regression oracle $\mathcal{R}^r$ is assumed to be an algorithm for the online regression problem with $\mathcal{Z} = \mathcal{X} \times[K],\mathcal{Y} = [0,1], \mathcal{H} = \mathcal{F}, \ell(y_1,y_2) = (y_1-y_2)^2$ so that when $\hat{y}_t$ is generated by $\mathcal{R}^r$, there exists some $\Regsq^r$ as a function of $T$ such that  \begin{align}\label{regret-reward-oracle}
 \Regsq(T; f^*) \leq \Regsq^{r}(T),\quad \forall f^*\in \mathcal{F}.
\end{align}

\noindent \textbf{Cost Regression Oracle} The cost regression oracle $\mathcal{R}^c$ is assumed to be an algorithm for the online regression problem with $\mathcal{Z} = \mathcal{X} \times[K],\mathcal{Y} = [0,1]^d,\mathcal{H} = \mathcal{G}, \ell(\bm y_1,\bm y_2) = \lVert \bm y_1-\bm y_2\rVert_\infty^2$ so that when $\bm \hat{y}_t$ is generated by $\mathcal{R}^c$, there exists some $\Regsq^c$ as a function of $T$ such that  \begin{align}\label{regret-cost-oracle}
 \Regsq(T; \bm g^*) \leq \Regsq^{c}(T),\quad \forall \bm g^*\in \mathcal{G}.
\end{align}
\subsection{Online Mirror Descent}
Our algorithm adopts an online primal-dual framework to tackle the challenge brought by knapsack constraints. Our strategy of updating the dual variable $\bm\lambda_t$ falls into the general online convex optimization (OCO) framework. In the OCO problem with parameter set $\Lambda$, adversary class $\mathcal{L}$ and time horizon $T$, the learner needs to choose $\bm\lambda_t \in \Lambda$ adaptively at each round $t$. After each choice, he will observe an adversarial convex loss $L_t \in \mathcal{L}$ and pay the cost $L_t(\bm\lambda_t).$ The goal of the learner is to minimize the cumulative regret: \begin{align*}
	\text{Reg}_{\text{OCO}}(T; \Lambda ,\mathcal L ): =  \sum_{t=1}^T L_t(\bm\lambda_t) - \min_{\bm\lambda\in \Lambda} \sum_{t=1}^T L_t(\bm\lambda) .
\end{align*}  
In our designed adversary, ${L}_t(\bm \lambda)=\langle \frac{B}{T}\cdot \bm 1-\bm c_{t,a_t}, \bm \lambda\rangle$ and minimizing $L_t$ corresponds to penalize the violation of the budget. We focus on the 
 following adversary class and parameter set:
{\small
 \begin{align*}
\mathcal{L} &= \{L(\bm\lambda):=  \langle \bm \theta,\bm \lambda\rangle : \bm\theta\in \mathbb{R}^d ,\lVert \bm \theta\rVert_\infty \leq 1 \} \\
\Lambda &= \{\bm \lambda \in \mathbb{R}^d: \bm \lambda \geq 0, \quad \lVert \bm \lambda \rVert_1\leq Z \},
 \end{align*}}
where $Z>0$ is the $\ell_1$ radius of $\Lambda$ to be determined. The online mirror descent (OMD) algorithm \citep{shalev2012online,hazan2016introduction} is a simple and fast algorithm achieving the optimal OCO regret in the non-euclidean geometry. OMD follows the update rule
{\small \begin{align}
  {\bm \lambda}_t &= \arg\min\limits_{\bm\lambda\in \Lambda}  \langle \nabla L_{t-1}(\bm\lambda_{t-1}),{\bm\lambda}\rangle + \dfrac{1}{\eta_t}D_h( {\bm \lambda},{ \bm \lambda}_{t-1}),
  \end{align}}
where $h$ is the metric generating function, and $D_h$ is the associated Bregman divergence. After adding a slack variable and re-scaling, the OCO problem over $\mathcal{L},\Lambda$ is equivalent to the OCO problem over 
 \begin{align*}
\mathcal{\tilde{L}} &= \{\tilde{L}(\bm\tilde{\bm \lambda}):=  \langle \bm \tilde{\bm \theta},\bm \tilde{\bm \lambda}\rangle : \bm\tilde{\bm \theta}\in \mathbb{R}^{d+1} ,\lVert \bm \tilde{\bm \theta}\rVert_\infty \leq Z, \bm \tilde{\bm \theta}_{d+1}=0 \} \\
\tilde{\Lambda} &= \{\bm \tilde{\bm \lambda} \in \mathbb{R}^{d+1}: \bm \tilde{\bm \lambda} \geq 0, \quad \lVert \bm \tilde{\bm \lambda} \rVert_1=1\},
 \end{align*}
which can be solved via the normalized Exponentiated Gradient (i.e., selecting $h$ as the negative entropy function). In this case, the OMD algorithm  has the following OCO regret guarantee \citep{shalev2012online,hazan2016introduction}: \begin{lemma} \label{lem-OCO}
Setting $\eta_t = \eta = O(\dfrac{\log d}{\sqrt{T}})$, the OMD yields regret
\begin{align}
    \text{Reg}_{\text{OCO}}(T; \Lambda ,\mathcal L )\lesssim Z \sqrt{T\log d }
\end{align}
\end{lemma}

As in the literature \citep{agrawal2016linear, castiglioni2022online} that introduces the online primal-dual framework, Lemma~\ref{lem-OCO} plays an essential role in controlling the regret induced by knapsack constraints in SquareCBwK.

\section{ALGORITHM AND THEORETICAL GUARANTEES}\label{sec-online-regret}
We are ready to present our main algorithm SquareCBwK for solving stochastic CBwK under the general realizability assumption. We first present the algorithm and theoretical results under $B=\Omega(T)$ regime for simplicity of algorithm design. Algorithm and analysis for more general choices of $B$ are presented in section~\ref{sec-general-B}.
\subsection{The SquareCBwK Algorithm}\label{sec-squarecbwk}
The main body of SquareCBwK has a similar structure as the SquareCB algorithm for contextual bandits \citep{foster2020beyond}, but with significant changes necessary to handle the knapsack constraints. SquareCBwK is presented in Algorithm~\ref{alg-squaredCBwK} with three key modules: prediction through oracles, arm selection scheme, and dual update through OMD.
\begin{algorithm}[h]
\caption{SquareCBwK}\label{alg-squaredCBwK}
\KwIn{Time horizon $\bm T$, total budget initial $\bm B$, learning rate $\bm \gamma>0$, online regression oracle for reward $\mathcal{R}^r$ and cost $\mathcal{R}^c$, radius $Z = T/B$ of the parameter set $\Lambda$ . }
Initialization:initialize $\mathcal{R}^r$, $\mathcal{R}^c$ and $\mathcal{R}^d$.\\
\For{$t=1,\dots,T$}{
Observe context $x_t$.\\
$\mathcal{R}^r$ predicts $\hat r_{t,a}$, $\mathcal{R}^c$ predicts $\hat{ \mathbf{c}}_{t,a}$, $\forall a \in [K]$ \\
Compute $\hat \ell_{t,a}: = \hat r_{t,a}+\lambda_t^T({B/T\cdot \bm 1- \mathbf{\hat c}}_{t, a}$), $\forall a\in [K]$.\\
Let $b_t = \arg\max\limits_{a \in [K]} \hat \ell_{t,a}$.\\
For each $a\neq b_t$, define $p_{t,a}:= \dfrac{1}{K  +\gamma\big (\hat \ell_{t,b_t}-\hat \ell_{t,a}\big)}$ and let $p_{t,b_t} = 1-\sum_{a\neq b_t}p_{t,a}$.\\
Sample arm $a_t \sim p_{t}$. Observe reward $r_{t,a_t} $ and cost $\bm c_{t,a_t}$.\\
\If{ $\exists j \in [d], \sum_{t^{'}=1}^{t} \bm c_{t,a_t}[j]\geq B-1$}{\textbf{Exit}}
Feed $\mathcal{R}^r$ with data $\{(x_t,a_t),r_{t,a_t})\}$ and $\mathcal{R}^c$  with data $\{(x_t,a_t),\bm c_{t,a_t}\}$.\\
Feed OMD with $\bm c_{t,a_t}$ and OMD updates $\lambda_{t+1} \in \Lambda$.\\
}
\end{algorithm}

 \noindent\textbf{Prediction through oracles}\quad 
 At each step, after observing the context, SquareCBwK will simultaneously access the reward oracle $\mathcal{R}^{r}$ and cost oracle $\mathcal{R}^{c}$ to predict the reward and cost of each action for this round. Then these two predicted scores are incorporated through the following Lagrangian to form a final predicted score $\bm{\hat {\ell}_t}$:
 \begin{align*}
    \hat{\ell}_{t,a} : = \hat{r}_{t,a} +\bm\lambda_t^T(\bm{1}\cdot B/T - \hat{\bm c}_{t,a}), \forall a\in [K].
    \end{align*}

 \noindent\textbf{Arm selection scheme}\quad After computing the predicted score $\bm{\hat \ell}_t$, we employ the probability selection strategy in \cite{abe1999associative}: We choose the greedy action score evaluated by $\hat\ell_t$ as the benchmark and select each arm $a$ with the probability $p_{t, a}$ that is inversely proportional to the gap between the arm's score and the benchmark. This strategy strikes a balance between exploration and exploitation: When the predicted score for an action is close to the greedy action, we tend to explore it with the probability roughly as $1/K$, otherwise with a very small chance.

\noindent\textbf{Dual update through OMD}\quad 
After choosing the arm $a_t$, the new data {$\{x_t,a_t,r_{t,a_t}, \bm c_{t,a_t}\}$} will be fed into $\mathcal{R}^c$, $\mathcal{R}^d$ and OMD. We then update the dual variable $\bm \lambda_t \in \Lambda$ successively through OMD. 

Compared with SquareCB in \cite{foster2020beyond}, we introduce three novel technical elements to deal with knapsack constraints: First, we apply a new cost oracle to generate the cost prediction. Next, to balance the reward and the cost prediction over time, we propose the predicted Lagrangian to construct a proper score function $\bm \hat{\ell}_{t}$.
Finally, we introduce OMD to update the dual variable so that the predicted scores $\bm \hat{\ell}_t$ adapt to the resource consumption throughout the process. Notably, the $l_1$ radius of the parameter set $\Lambda$ is carefully chosen to be $T/B$ since we expect $\bm \lambda_t$ can capture the sensitivity of the optimal static policy \eqref{eq-OPT-static} to knapsack constraints violations over time. Specifically, if we increase the budget $B$ by $\varepsilon$, the increased reward over T rounds is at most $\frac{T\text{OPT}}{B}\varepsilon$. This observation suggests that $Z$, the radius of $\Lambda$ should be at least $\frac{T\OPT}{B}$. On the other hand, $Z$ should be of constant level so that OMD achieves optimal regret bounds $O(\sqrt{T})$ by Lemma~\ref{lem-OCO}. Therefore, setting $Z=\frac{T\text{OPT}}{B}$ will be the desired choice. In practice, since $\OPT$ is unknown, we need to estimate $Z$ so that $ \frac{T\OPT}{B}\leq  Z \lesssim \frac{T\OPT}{B}$. The regime $B=\Omega(T)$ guarantees that $\frac{T}{B}$ is approximately $\frac{T\text{OPT}}{B}$, without the need to further estimate $\OPT$. This is why we set $Z=\frac{T}{B}$ in SquareCBwK. We will discuss estimating $\OPT$ when $B=\Omega(T)$ fails to hold in Section~\ref{sec-general-B}.

Now we state the theoretical guarantee of SquareCBwK:
\begin{theorem}\label{thm-regret-upper-bound}
Considering the regime $B =\Omega(T)$ under Assumption~\ref{assumption_realizable}, if the output of $\mathcal R^{r}$ and $\mathcal R^{c}$ satisfy \eqref{regret-reward-oracle} and \eqref{regret-cost-oracle}, respectively,  
denote  {\footnotesize $$\gamma=\sqrt{K T /\left(\operatorname{Reg}_{Sq}^{r}(T)+(T/B+1)^2\operatorname{Reg}_{Sq}^{c}(T)+4\log (2 T)\right)},$$ }  then SquareCBwK achieves the regret
{ \begin{align*}
 \operatorname{Reg}(T) \lesssim & \sqrt{K T \cdot \big (\operatorname{Reg}_{\mathrm{Sq}}^{r}(T)++ \log (dT) \big)}\\
 & + (\frac{T}{B}+1)\sqrt{KT\operatorname{Reg}_{\mathrm{Sq}}^{c}(T)}
\end{align*}}
\end{theorem}

Compared to the result in \cite{foster2020beyond} for the unconstrained contextual bandit problem, our regret bound has an additional dependency on $\text{Reg}^c_{\text{sq}}(T),$ which is a natural outcome under the budget-setting. Moreover, with knapsack constraints, the optimal static policy will be a distribution over actions rather than pulling a single optimal arm. As a result, in the proof of Theorem~\ref{thm-regret-upper-bound}, we need to adapt the argument related to the probability selection strategy in \cite{foster2020beyond} to this significant change and derive a lower bound for the total expected predicted scores. Then we split the total expected reward from the expected Lagrangian scores and control the regret incurred by the early stopping time $\tau < T$ using the regret of OMD and special radius selection $Z$. We relate expectation with realization to obtain the final regret bound.


Theorem~\ref{thm-regret-upper-bound}  provides an upper bound for the regret of stochastic CBwK by reducing it to regression, a basic supervised learning task. We further show such a reduction is optimal for various function classes in Section~\ref{sec-optimality} where 
 we design a novel way to derive lower bounds for the regret of CBwK  that match the upper bound given by Theorem~\ref{thm-regret-upper-bound}. With this optimal reduction, our framework is quite general and flexible: it can be instantiated with any available efficient and {optimal} oracles for general $\mF$ and $\mG$ (we also allow $\mG$ to be different from $\mF$), then the optimal regret for CBwK can be directly given by Theorem~\ref{thm-regret-upper-bound}. 



\subsection{Lower Bound Results}\label{sec-optimality}

To demonstrate the optimality of Theorem~\ref{thm-regret-upper-bound},  we need to discuss whether the dependency on $\Regsq^r$ and $\Regsq^c$ is tight.
In the unconstrained setting, \cite{foster2020beyond} shows the tightness result of $\Regsq^r$ for a wide range of nonparametric classes $\mathcal{F}$.
 Since the contextual bandit problem can be seen as a special case of CBwK problem with $B = T$ and $d = 1$, the optimality results in \cite{foster2020beyond} can be utilized to show the tight dependency on $\Regsq^r$ in Theorem~\ref{thm-regret-upper-bound}, which
also indicates the tight dependency on ${\Regsq^c}$  when $\mathcal G$ shares the same structure as $\mathcal F$ since the term $\Regsq^c$ can be absorbed by $\Regsq^r$ in this case. 
However, when the complexity of $\mG$ is much higher than $\mF$, the lower regret bound in the unconstrained setting will be loose compared with the upper bound in Theorem~\ref{thm-regret-upper-bound}. To close this gap, we obtain a general result that establishes lower bounds for CBwK concerning cost function class $\mathcal{G}$ based on its in-separation property. 

To present our result, for a general function class $\mathcal{H},$ we first introduce a new concept that characterizes the difficulty of the \textbf{unconstrained} contextual bandit problems with  $\mathcal{H}$ as the reward function class:
\begin{definition}[$\alpha$-inseparable class]\label{def-gamma-inseparable}
For a fixed time horizon $T$, we say a given function class $\mathcal{H}$ of functions from $\mathcal{X} \times \{0,1\}$ to $ [0,1]$ is $\alpha$-inseparable with respect to some $P_{\mX}$ over $\mX$,  if  there exist $\mathcal{U}_{\alpha}\subset \mathcal{H}$ and an absolute constant $c>0$ independent of $T$  such that
\vspace{-0.3cm}
\begin{enumerate}
        \item  For all $h\in \mathcal{U}_{\alpha}$,  denoting $a^*(x;h) = \mathop{\text{argmax}}\limits_{{a \in \{0,1\}}}h(x,a),$ it holds that {  \begin{align*}
      &\E_{P_{\mX}}[h(x,a^*(x;h))] -\alpha\\
       \geq &\E_{P_{\mX}}[h(x,1-a^*(x;h))] 
       \geq \frac{1}{4}
        \end{align*} }
        
        \vspace{-0.3cm}
        
        \item For every dynamic policy $\pi,$ there exist some $h \in \mathcal{U}_{\alpha}$ and a distribution of $r_{t,a}$ with $\E[r_{t,a}\lvert x_t = x ] = h(x,a)$, such that
        \begin{align*}
        \E_{x,r}[\sum_{t=1}^T \bm{1}\{ \pi_t(x_t)\neq a^*(x_t;h)  \} ] & > c T.
        \end{align*}  
    \end{enumerate}
    \end{definition}
\vspace{-0.3cm}

We refer to the above properties $1$ and $2$ as the $\alpha$-inseparable property since they indicate that even the optimal action outperforms the sub-optimal action with a reward gap larger than $\alpha$, there exists no dynamic policy that can distinguish the optimal action without exploring at least $\Omega(T)$ steps. One straightforward observation from Definition~\ref{def-gamma-inseparable} is that, if $\mathcal{H}$ is $\alpha$-inseparable, then
\begin{align*}
  \RegCB(T;\mathcal{H}) \gtrsim \alpha T,
\end{align*} 
where $\RegCB(T;\mathcal{H})$ is the optimal minimax regret bound of unconstrained CB problem with $\mathcal{H}$ as the expected reward class. Here the minimum is taken over all possible dynamic policies and the maximum is taken over all possible expected reward function $h^*\in \mathcal{H}$, all $P_{\mX}$ over $\mX$ and all conditional distributions $P_r$ of rewards with $E[r_{t,a}\lvert x] = h^*(x,a).$ 

Indeed, the $\alpha$-inseparable property has been applied implicitly for many classes in previous works to derive tight lower bounds for $\RegCB(T;\mathcal{H})$, e.g., linear class \citep{chu2011contextual}, H\"older class \citep{rigollet2010nonparametric} and general nonparametric class \citep{foster2020beyond}. 

Now we state our main theorem that establishes regret lower bounds for CBwK with $\alpha$-inseparable classes:
\begin{theorem}\label{thm-lb-general} Consider the class of CBwK problems with two non-null arms, $d = 1$, the reward class $\mathcal{F}$ and the cost class $\mathcal{G}$, where $\mathcal{G}$ is $\alpha$-inseparable with respect to some $P_0$ over $\mathcal X$, and there exists some $f_0 \in \mathcal{F}$ such that $f_0(x,0) = f_0(x,1) = \Omega(1)$ a.s. under $P_0$.
Then there exists $T$ and $B^*>T/8$ such that for every dynamic policy $\pi$, there exists $f^*\in \mathcal{F},g^*\in \mathcal{G}$ and distribution $P_{\mX}^*$ of context, distribution $P_{r}^*,P_{c}^*$ on rewards and costs with $\E[r_{t,a}\lvert x_t] = f^*(x_t,a)$, $\E[{c}_{t,a}\lvert x_t] = g^*(x_t,a)$ such that when one runs $\pi$ on the CBwK instance with budget $B^*$, context distribution $P_{\mX}^*$, reward distribution $P_r^*$ and cost distribution $P_c^*$, it holds that $$\text{Reg}_\pi (T) \gtrsim \max\{\RegCB(T;\mathcal{F}), \alpha T \}.$$
\end{theorem}

The condition on the existence of $f_0$ is a technical assumption that ensures one can always construct an instance with the achieved reward independent of the action taken. For such instance, the regret of a policy is only determined by its over-cost of the resource, which can be lower bounded by $\alpha T$.
As discussed above, the $\alpha T$ term is a lower bound of $\RegCB(T;\mathcal{G})$, which is the CB lower bound when $\mG$ is the reward class. 
In this sense, Theorem~\ref{thm-lb-general} reveals that the lower bound of a CBwK problem with reward class $\mathcal{F}$ and cost class $\mathcal{G}$ can be established by just considering unconstrained CB regret bounds with reward classes $\mF$ or $\mG$. 
\subsection{Applications}\label{sec-example}

\begingroup 
\renewcommand{\arraystretch}{1.4} 
\begin{table*}[ht!]
	\centering
	  \caption{Regret bound results of SquareCBwK with different instantiations when $B = \Omega(T)$ .}
	  \label{tbl-table-comparison}
   \begin{threeparttable}
	  \resizebox{\textwidth}{!}{
	  \begin{tabular}{|c|c|c|c|c|c|}
	  \hline
	  $\mathcal{F}$ & $\tilde{\mathcal{G}}$  & $\mathcal{R}^r$ and $\tilde{\mathcal{R}}^c$  & Previous Results &  Regret by Thm~\ref{thm-regret-upper-bound}  &  Lower Bound by Thm~\ref{thm-lb-general} \\
	  \hline
	$m_1$-dim Linear  & $m_2$-dim Linear &   \multirow{2}{*}{Newtonized GLMtron\tnote{$*$} }     & $\tilde{O}((m_1+m_2)\sqrt{T})\tnote{$\dagger$}$ &\multirow{2}{*}{$\tilde{O}(\sqrt{(m_1 +dm_2)K T})$} &  \multirow{2}{*}{$\Omega(\sqrt{(m_1+m_2)T} )$}\\
 	  \cline{1-2} \cline{4-4}
	    $m_1$-dim Generalized Linear & $m_2$-dim Generalized Linear  &   &\multirow{2}{*}{N.A.} & & \\
	  \cline{1-3} \cline{5-6}
$p_1$-Nonparametric & $p_2$-Nonparametric & Vovk's Aggregation\tnote{$\S$} &  & $\tilde{O}\big ((KT)^{\frac{1+p_1}{2+p_1}}+\sqrt{d}(KT)^{\frac{1+p_2}{2+p_2}} \big )$ & $\tilde{\Omega}(T^{\frac{1+p_1}{2+p_1}}+T^{\frac{1+p_2}{2+p_2}} \big )$ \\
	  \hline
	  \end{tabular}
	  }
   {\footnotesize
    ${}^*$ Proposition~3, \cite{foster2020beyond} ;
    ${}^\S$ Theorem~3, \cite{foster2020beyond} ; ${}^\dagger$ Theorem~2, \cite{agrawal2016linear} (while \cite{agrawal2016linear} study the setting $m_1 = m_2,$ it is straightforward to extend their result to general $m_1,m_2.$)
   }
   \end{threeparttable}
  \end{table*}
\endgroup

In this section, we instantiate SquareCBwK with different oracles for the generalized linear class and nonparametric function classes, respectively. We show that Theorem~\ref{thm-regret-upper-bound} and Theorem~\ref{thm-lb-general} can provide tight regret upper and lower bounds for these classes. The regret bounds and selection of oracles are summarized in Table~\ref{tbl-table-comparison}. As far as we know, no previous results study CBwK beyond the linear setting or consider the tightness of lower bounds when $\mathcal{F}$ and $\mathcal{G}$ are different even for the linear case.
 Although here we focus on the generalized linear and nonparametric classes, SquareCBwK can also be instantiated with available oracles of other function classes, e.g., the kernel classes and uniform convex Banach spaces discussed in section~2.3 of \cite{foster2020beyond}.

For the examples considered in this section, the vector-valued function class $\mathcal{G}$ is a product of $d$ same $[0,1]$-valued function class $\tilde{\mathcal{G}}$, i.e., \begin{align*}
	\mathcal{G} =\tilde{\mathcal{G}}^d=  \{\bm g: \bm g = (g_1,\dots,g_d), g_i \in \tilde{\mathcal{G}}, \forall i \in [d]  \}.
\end{align*}
In this case, we can construct the online regression oracle $\mathcal{R}^c$ satisfying $\eqref{regret-cost-oracle}$ with $\Regsq^c(T) \leq  d\widetilde{\Reg}_{\text{sq}}^c(T)$ from any oracle $\tilde{\mathcal{R}}^c$ over $\tilde{\mathcal{G}}$ satisfying $\eqref{regret-cost-oracle}$ with $\widetilde{\Reg}_{\text{sq}}^c(T)$. We provide such construction in Appendix~\ref{sec-appendix-example-proof}. Here we only specify $\tilde{\mathcal{G}}$ and $\tilde{\mathcal{R}}^c$ in the following examples.

\noindent\textbf{Generalized Linear CBwK}\quad In the generalized linear CBwK setting, there exist known feature maps $\phi_1: \mathcal{X} \times [K] \to \mathbb B^{m_1}$, $\phi_2: \mathcal{X} \times [K] \to \mathbb B^{m_2}$, where $\mathbb B^m$ is the unit $\ell_2$ ball in $\mathbb{R}^m$, and link functions $\sigma_i: [-1,1] \to [0,1]$ satisfying $0<c \leq \sigma'_{i}(x) <1 $  for $i = 1,2.$ The function classes are selected as \begin{align*}
	\mathcal{F} &= \{(x,a) \to \sigma_1(\langle \theta,\phi_1(x,a) \rangle), \theta\in \mathbb{B}^{m_1} \},\\
	\tilde{\mathcal{G}} &= \{(x,a) \to  \sigma_2(\langle \theta,\phi_2(x,a) \rangle), \theta\in \mathbb{B}^{m_2} \}.
\end{align*}

\vspace{-0.3cm}

In this case, selecting $\mathcal{R}^r,\tilde{\mathcal{R}}^c$ as the Newtonized GLMtron oracle (\cite{foster2020beyond}, Proposition~3)  achieves{ \small $\Regsq^r(T) \lesssim m_1 \log(T), {\Reg}_{\text{sq}}^c(T) \lesssim d m_2 \log(T)$}. Then Theorem~\ref{thm-regret-upper-bound} implies the $O(\sqrt{(m_1+dm_2)KT})$ regret of SquareCBwK. On the other hand, we can verify $\tilde{\mathcal{G}}$ is $\sqrt{m_2/T}$-inseparable (see Appendix~\ref{sec-appendix-example-proof}), and $\RegCB(\mathcal{T;F})\gtrsim \sqrt{m_1 T}$. Then applying Theorem~\ref{thm-lb-general} leads to a $\Omega(\sqrt{(m_1+dm_2)KT})$ lower bound when $d = 1$. Our upper bound and lower bounds imply the optimality of SquareCBwK with respect to $T$ and $m_1,m_2$. Moreover, besides the $O(\sqrt{(m_1+m_2d)KT})$ achieved by the GLMtron oracle, we can also select the Online Gradient Descent oracle (\cite{foster2020beyond}, Proposition~2) for SquareCBwK to get $\tilde{O}(\sqrt{Kd}T^{3/4})$ regret, which has worse dependency on $T$ but is independent of dimension.

By selecting $\sigma$ as the identical map, and assuming in additional the ranges of $\phi_1,\phi_2$ lie in $\mathbb{R}^{m_1}_+,\mathbb{R}^{m_2}_+$, the generalized linear class covers the linear model as a special case. Thus our algorithm also applies to linear CBwK. Compared with the $\tilde{O}((m_1+m_2)\sqrt{T})$ regret achieved in \cite{agrawal2016linear}, SquareCBwK with Newtonized GLMtron oracle has an additional dependency on $K,d$, but improves the dependency of $m_1,m_2$ to $\sqrt{m_1},\sqrt{m_2}$ , which matches the established lower bound. In Appendix~\ref{simulations}, we perform simulations for linear CBwK, comparing the dependencies on time horizon $T$, dimension $m$, and number of arms $K$ of SquareCBwK utilizing Newtonized GLMtron and Online Gradient Descent oracles with those of the LinUCB in \cite{agrawal2016linear}. These simulations provide numerical verification of the aforementioned theoretical guarantees.

\noindent \textbf{CBwK with nonparametric function classes}\quad 
We say a function class $\mathcal{W}$ is in the nonparametric regime if its metric entropy ${H}(\mathcal{W},\varepsilon)$ scales in $\varepsilon^{-p}$ for some $p>0$ \citep{rakhlin2017empirical}. More precisely, we say a class $\mathcal{W}$ is $p$-nonparametric if
	{\small $ {H}(\mathcal{W},\varepsilon) \lesssim \varepsilon^{-p}, \quad \forall \varepsilon >0.$}
 We make the additional assumption that both $\mathcal{F}$ and $\mathcal{G}$ tensorizes: There exists $p_1$-nonparametric and $p_2$-nonparametric classes $\mathcal{W},\mathcal{V}$ 
 so that \begin{align*}
	\mathcal{F} &= \{ f:(x,a) \to   w_a(x), w\in \mathcal{W}  \},\\
	\mathcal{G} &= \{ \bm g:(x,a) \to   \bm v_a(x), \bm v\in \mathcal{V}  \}.
\end{align*}
The $p$-nonparametric class is general enough to cover many classes including the H\"older class considered in most nonparametric CB literature \citep{slivkins2011contextual,rigollet2010nonparametric,hu2020smooth}.
 In this case, an Vovk's aggregation based oracle is proposed in \cite{foster2020beyond} for $\mathcal{F}$ and $\tilde{\mathcal{G}}$ with $\Regsq^r(T)=\tilde{O}((KT)^{1-\frac{2}{2+p_1}}),\Regsq^c(\mathcal{G}) = \tilde{O}(d(KT)^{1-\frac{2}{2+p_2}})$, thus Theorem~\ref{thm-regret-upper-bound} implies the $\tilde{O}\big ((KT)^{\frac{1+p_1}{2+p_1}}+\sqrt{d}(KT)^{\frac{1+p_2}{2+p_2}} \big )$ regret guarantee of SquareCBwK. By verifying that $\tilde{\mathcal{G}}$ is $T^{1-\frac{2}{2+p_2}}$-inseparable (see Appendix~\ref{sec-appendix-example-proof}) and applying Theorem~\ref{thm-lb-general}, we also establish a tight lower bound $\tilde{\Omega}(T^{\frac{1+p_1}{2+p_1}}+T^{\frac{1+p_2}{2+p_2}})$. 
 Our upper and lower bound results imply the \textit{universality} of the SquareCBwK: for every general nonparametric $\mathcal{F}$ and $\tilde{\mathcal{G}}$ there always exist choices of $\mathcal{R}^r$ and $\tilde{\mathcal{R}}^c$ such that SquareCBwK achieves the optimal regret bound with respect to $T$ and the complexity parameters $p_1,p_2$ of $\mathcal{F},\tilde{\mathcal{G}}$.

\section{ALGORITHM WITH RELAXED ASSUMPTION ON $\bm B$}\label{sec-general-B}

\begingroup 
\renewcommand{\arraystretch}{1.4} 
\begin{table*}[ht!]
	\centering
 	  \caption{Regret bound results of TwoStage-SquareCBwK with different instantiation when $B = o(T)$ .}
	  \label{tbl-table-B}
   \begin{threeparttable}
	  \resizebox{\textwidth}{!}{
	  \begin{tabular}{|c|c|c|c|c|}
	  \hline
	  $\mathcal{F}$ and $\tilde{\mathcal{G}}$  & Theorem &  Regret & Requirement on $B$ & Phase I length  \\
	  \hline
	\multirow{3}{*}{$m$-dim Linear} &  Theorem~3, \cite{agrawal2016linear}     & \multirow{2}{*}{$\tilde{O}((\frac{T\OPT}{B}+1)m\sqrt{T} )$} & {$\tilde{\Omega}(mT^{3/4})$} & $O(m\sqrt{T})$ \\
 	  \cline{2-2} \cline{4-5}
    &  Corollary~1, \cite{sivakumar2022smoothed}     &  & {${\Omega}(m^{2/3}T^{3/4})$} &  $\tilde{O}(m^{2/3}\sqrt{T})$ \\
    \cline{2-5}
	     & Theorem~\ref{thm-finer-regret}  & $\tilde{O}((\frac{T\OPT}{B}+1)\sqrt{KdmT})$  & $\tilde{\Omega}\big((md)^{1/3}(KT)^{3/4} \big)$  & $\tilde{O}((md)^{1/3} \sqrt{KT})$ \\
	  \cline{1-5} 
$p$-Nonparametric &  Theorem~\ref{thm-finer-regret}   & $\tilde{O}\big((\frac{T\OPT}{B}+1)\sqrt{d}(KT)^{\frac{1+p}{2+p}}  \big)$  &  $\tilde{\Omega}( d^{\frac{2+p}{6+2p}} (KT)^{\frac{3+p}{4+2p}})$ & $\tilde{O}(d^{\frac{2+p}{6+2p}} (KT)^{\frac{1+p}{2+p}})$  \\
	  \hline
	  \end{tabular}
	  }
   \end{threeparttable}
  \end{table*}
\endgroup

In this section, we aim to relax the assumption $B = \Omega(T)$. When $B = o(T)$, applying the Theorem~\ref{thm-regret-upper-bound} in this scenario will lead to a sub-optimal result due to its dependency on the $\frac{T}{B}$ factor. Indeed, the true dependency should be $\frac{T\OPT}{B}$ as we discussed before. We can get rid of this $\frac{T}{B}$ factor by replacing the radius of $\Lambda$ with a factor $Z$ that approximates $\frac{T\OPT}{B}$ better. We propose a two-stage algorithm, in which $\frac{T\OPT}{B}$ is approximated by $Z$ in the first phase, and in the second phase Algorithm~\ref{alg-squaredCBwK} is run with $Z$. 
 
\begin{algorithm}[h]
\caption{TwoStage-SquareCBwK}\label{alg-squaredCBwK-Z}
\KwIn{ $\bm T$, $\bm B$, 
$T_0$, estimate oracles $\mathcal{R}_{est}^r$, $\mathcal{R}_{est}^c$,
other parameters in Algorithm~\ref{alg-squaredCBwK}. }
\textbf{Initialization:} $\mathcal{D}^r_{0,a},\mathcal{D}^c_{0,a} = \emptyset, \forall a\in [K].$ \\
\For{$a=1,\dots,K$}{
\For{$t=(a-1)T_0+1,\dots, aT_0$ }
{Play arm $a$ and observe $r_t,\bm c_t.$\\
 $\mD_{0,a}^r = \mD_{0,a}^r \cup \{r_t\}$ ,\\
 $\mD_{0,a}^c = \mD_{0,a}^c\cup\{ \bm c_t\}$.
}
Set {\footnotesize $\hat{f}_0(\cdot,a) $ } the output of {\footnotesize $\mathcal{R}_{est}^r$} with input  {\footnotesize $\mD_{0,a}^r$}.\\
Set {\footnotesize $\hat{\bm g}_0(\cdot,a) $ } the output of {\footnotesize $\mathcal{R}_{est}^c$} with input  {\footnotesize $\mD_{0,a}^c$}.}
\For{$t=KT_0+1,\dots, (K+1)T_0$}{
Pull arm $a_t$ arbitrarily.
}
Solve the linear programming \eqref{eq-Z-est-LP} and get {\small $\widehat{\OPT}(T_0)$}.\\
Set the radius of $\Lambda$ to $Z$ as computed in Lemma~\ref{lem-opt-error}.\\
Set remaining resource as $\bm B^{'}=\bm B-(K+1)T_0\bm{1}$. \\
Run Algorithm \ref{alg-squaredCBwK} for $T-(K+1)T_0$ rounds with remaining budget $\bm B^{'}$ and  $Z$.
\end{algorithm}

 While such a two-stage design has been used in the linear setting \citep{agrawal2016linear,sivakumar2022smoothed}, their designs depend on the special structure of linear classes and the self-normalized martingale concentration, which cannot be extended to general classes. To estimate $Z$ for more general $\mathcal{F},\mathcal{G}$, we introduce the statistical regression oracles $\mathcal{R}_{est}^r,\mathcal{R}_{est}^c$ with the following assumptions:
\begin{assumption} \label{assumption-relaxB}
For any $a$, given a dataset $\mathcal{D}_a$ containing $M$ i.i.d. samples {\small $\{(x_i,r_{i,a},\bm{c}_{i,a})\}_{i=1}^M \sim P_{\mX}$}, the output $\hat{f},\hat{\bm g}$ of $\mathcal{R}_{est}^r,\mathcal{R}_{est}^c$ with input $\mathcal{D}_a$  satisfy \begin{align*}
    \E_{x\sim P_{\mathcal X}}\big[ (\hat{f}(x,a)-f(x,a))^2 \big] \leq \mathcal{E}_{\delta}(M;\mathcal{F}),\\
    \E_{x\sim P_{\mathcal X}}\big[ \lVert \hat{\bm g}(x,a)-\bm g(x,a)\rVert^2_\infty \big] \leq \mathcal{E}_{\delta}(M;\mathcal{G}).
\end{align*}
 with probability at least $1-\delta$.
\end{assumption} 

 Indeed, as long as there exists an online regression oracle satisfying \eqref{regret-reward-oracle} and \eqref{regret-cost-oracle}, we can apply the standard online-to-batch (OTB) method to construct a statistical regression oracle that satisfies 
{\small $\mathcal{E}_{\delta}(M;\mathcal{F})\lesssim \frac{\Reg_{sq}^r(M) \log({1}/{\delta})}{M}$ and $ \mathcal{E}_{\delta}(M;\mathcal{G})\lesssim \frac{\Reg_{sq}^c(M) \log({1}/{\delta}) }{M}.$} We leave the construction of OTB oracle and the proof of such estimation error to Appendix~\ref{sec-OTB}.

The Stage $1$ of Algorithm ~\ref{alg-squaredCBwK-Z} includes the first $(K+1)T_0$ rounds 
to estimate $\frac{T\OPT}{B}$. For the first $KT_0$ rounds, we pull each arm evenly for $T_0$ times regardless of the context and gather outcomes. We then use the oracles over the collected data to generate predictors $\hat{f}_0$ and $\hat{\bm g}_0$. For the next $T_0$ rounds, we collect the contexts and pull arms arbitrarily. The contextual information in the latter $T_0$ rounds is used to estimate $\OPT$ by solving the following linear programming over $(\Delta^{K})^{T_0}:$ 
 {\footnotesize\begin{equation}\label{eq-Z-est-LP}
 \begin{aligned}
	\max_{p \in  (\Delta^{K})^{T_0}}&\dfrac{1}{T_0} \sum_{t\in \mT_0} \sum_{a\in [K]} p_{t,a} \hat{f}_{0}(x_t,a)\\
 \text{subject to}&  \dfrac{1}{T_0} \sum_{t\in \mT_0}\sum_{a\in [K]} p_{t,a} \hat{\bm{g}}_{0}(x_t,a) \leq \dfrac{\bm B}{T}+ 2 \mathcal{M}(T_0),
\end{aligned}
\end{equation}}
\noindent where we denote {\footnotesize \begin{align*}
    \mT_0&: = \{t: KT_0+1 \leq t\leq (K+1)T_0\},\\
    \mathcal{M}(T_0)&:= \sqrt{K(\mathcal{E}_{T_0}(\mF)+ d\mathcal{E}_{T_0} (\mG) )+ 4\frac{\log(Td)}{T_0}}, 
\end{align*}}
\noindent and {\footnotesize $\mathcal{E}_{T_0}(\mathcal{F}):= \mathcal{E}_{1/T}(T_0;\mathcal{F}) ,  \mathcal{E}_{T_0}(\mathcal{G}):= \mathcal{E}_{1/T}(T_0;\mathcal{G}).$} 

\allowdisplaybreaks

The above linear programming can be seen as an empirical approximation to the static programming (1).
We further present a lemma that gives an estimator of {$T{\OPT}/{B}$} whose error is bounded by the estimation error of {\small $\mathcal{R}_{\text{est}}^r$} and {\small $\mathcal{R}_{\text{est}}^c$}: 
\begin{lemma}\label{lem-opt-error} Denoting the optimal value of \eqref{eq-Z-est-LP}  by $\widehat{\OPT}(T_0)$ and set $Z = \frac{T}{B}(\widehat{\OPT}(T_0)+  \mathcal{M}(T_0) )$,
 we have with probability at least $1-O(1/T^2)$,
	{\footnotesize  \begin{align*}
	  \dfrac{T\OPT}{B} \leq  Z \leq (\dfrac{6T\mathcal{M}(T_0)}{B}+1)( \dfrac{T\OPT}{B} + 1)
\end{align*}}
In particular, $Z\lesssim \frac{T\OPT}{B}$ if {  $B = \Omega \big( {T(\mathcal{M}(T_0)}\big ).$}
\end{lemma} 

With the estimation error guarantee of $Z$, we  can obtain the following regret guarantee of Algorithm~\ref{alg-squaredCBwK-Z}
by a modification of the proof of Theorem~\ref{thm-regret-upper-bound}:
\begin{theorem}\label{thm-finer-regret} 
When $B> \max\{(K+2)T_0 , T \mathcal{M}(T_0) \}$, under Assumption~\ref{assumption_realizable} and Assumption~\ref{assumption-relaxB}, Algorithm~\ref{alg-squaredCBwK-Z} achieves the regret{\begin{align*}
\Reg(T) \lesssim  & (\dfrac{T\OPT}{B}+1) \sqrt{KT [\text{Reg}_{sq}^r (T)+\text{Reg}_{sq}^c (T)+1] }\\
& + (\dfrac{T\OPT}{B}+1) K T_0
\end{align*}}
\end{theorem}
The regret bound in Theorem~\ref{thm-finer-regret} reveals the lower bound requirement on $B$ and provides a guideline on selecting the exploration length $T_0$  for general $\mathcal{F},\mathcal{G}$. We summarize results 
of $T_0$ selection and requirements of $B$ when $\mathcal{F}$ and $\tilde{\mathcal{G}}$ are linear classes and nonparametric classes in Table~\ref{tbl-table-B} and leave the proof to Appendix. While Theorem~\ref{thm-regret-upper-bound} can apply to more general $\mathcal{F}$ and $\mathcal{G}$ as in Table~\ref{tbl-table-comparison}, we present the result with $\mathcal{F} = \mathcal{G}$ for simplicity. In particular, in the linear CBwK setting, 
both the regret bound and the requirement on $B$ have a better dependency on $m$ than previous results. On the other hand, compared with previous results in the linear setting, our algorithm has an additional dependency on $K$ in the Phase I length. The reason for such dependency is that we take a uniform exploration for general $\mathcal{F}$ and $\mathcal{G}$ instead of an adaptive exploration procedure as in the linear setting \citep{agrawal2016linear}. Developing more efficient algorithms of estimating $\widehat{\OPT}$ for general function classes is an interesting future direction.
\section{CONCLUSION}

In this paper, we present a new algorithm for CBwK prob-\\lem with general reward and cost classes. 
Our algorithm provides a reduction from CBwK problems to online regression problems. 
\text{By providing the regret upper bound} that matches the lower bound,  we demonstrate the optimality of our algorithm for various function classes. 

There are several future directions that can be explored. 
First, our assumption on the cost regression oracle will implicitly lead to an extra $\sqrt{d}$ factor in many examples. One future direction is to relax this assumption to improve the dependency on $d$. Another related open question is whether a similar reduction from CBwK problems to offline regression problems is possible as in the CB setting \citep{simchi2021bypassing}. Finally, extensions of our framework to the misspecified setting \citep{foster2021misspec} and large action space settings \citep{zhu2022large} are also promising directions to explore.

\section*{Acknowledgements}
The authors would like to thank Xiaocong Xu for helpful discussions and reviewers for valuable suggestions. 
This work was supported by HKUST IEG19SC04, the Project of Hetao Shenzhen-HKUST Innovation Cooperation Zone HZQB-KCZYB-2020083, the Guangdong-Hong Kong-Macao Joint Laboratory for Data-Driven Fluid Dynamics, and Hong Kong Research Grant Council (HKRGC) Grant 16214121, 16208120.

\bibliographystyle{apalike}
\bibliography{ref}
\appendix

%

%

\onecolumn

\section{PROOF OF RESULTS IN SECTION~\ref{sec-online-regret}}

\subsection{Proof of Theorem~\ref{thm-regret-upper-bound}}
We would prove the following more general form of Theorem~\ref{thm-regret-upper-bound}
\begin{theorem}\label{thm-general-Z-upper-bound}
    Under Assumption~\ref{assumption_realizable}, suppose  $Z\geq T\OPT/B$ and the dual-update step is running over $\Lambda$ with radius $Z$, if the output of $\mathcal R^{r}$ and $\mathcal R^{c}$ satisfy \eqref{regret-reward-oracle} and \eqref{regret-cost-oracle}, respectively,  
denote  {\footnotesize $$\gamma=\sqrt{K T /\left(\operatorname{Reg}_{Sq}^{r}(T)+Z^2\operatorname{Reg}_{Sq}^{c}(T)+4\log (2 T)\right)},$$ }  then SquareCBwK achieves the regret
{\footnotesize \begin{equation*}
 \operatorname{Reg}(T) \lesssim  (Z+1)\sqrt{K T \cdot \big (\operatorname{Reg}_{\mathrm{Sq}}^{r}(T)+\operatorname{Reg}_{\mathrm{Sq}}^{c}(T)+ \log (dT) \big)}. 
\end{equation*}}
\end{theorem}
In particular, since it always hold that $T\OPT/B\leq T/B$, Theorem~\ref{thm-regret-upper-bound} is a special case of Theorem~\ref{thm-general-Z-upper-bound} with setting $Z = T/B$. 

\begin{proof}[Proof of Theorem~\ref{thm-general-Z-upper-bound}]
Denote $\ell_t(a): = f^*(x_t,a) +\bm\lambda_t^T\big ( B/T \cdot\bm 1-\bm g^*(x_t,a) \big )$ and $p^*(x)$ the optimal solution of static programming~\eqref{eq-OPT-static}$,p^{*}_{t,a}: =p^{*}_a(x_t)$,  then we have \begin{align*}
	&\sum_{t=1}^\tau  [\sum_{a'=1}^{K} p_{t,a'}^* \ell_t(a') -  \sum_{a=1}^K p_{t,a} \ell_t(a) ]  \\
	  = &\sum_{t=1}^\tau  \sum_{a'=1}^K p_{t,a'}^*[ \ell_t(a') - \sum_{a=1}^K p_{t,a}\ell_t(a)   ] \\
	  =& \sum_{t=1}^\tau  \sum_{a'=1}^K p_{t,a'}^*\bigg( \sum_{a=1}^K p_{t,a}\big[\ell_t(a') - \ell_t(a) -\dfrac{\gamma}{4}\big(\hat{\ell}_t(a) - \ell_t(a) \big )^2+ \dfrac{\gamma}{4}  \big(\hat{\ell}_t(a) - \ell_t(a) \big )^2 \big] \bigg),
	 \end{align*}
where we also denote $\hat{\ell}_t(a):= \hat{\ell}_{t,a}=\hat r_{t,a}+\bm \lambda_t^T({B/T\cdot \bm 1- \mathbf{\hat c}}_{t, a})$.  Notice that denote $a_t^*: = \text{argmax}_{a\in [K]} \ell_t(a)$, we have \begin{align*}
   \sum_{a=1}^K p_{t,a} [\ell_t(a') - \ell_t(a) -\dfrac{\gamma}{4}\big(\hat{\ell}_t(a) - \ell_t(a) \big )^2]
    &\leq 
    \sum_{a=1}^{K} p_{t,a} [\ell_t(a_t^*) - \ell_t(a) -\dfrac{\gamma}{4}\big(\hat{\ell}_t(a) - \ell_t(a) \big )^2] \leq \dfrac{2K}{\gamma}.
\end{align*}
Where in the second inequality we used Lemma~3 in \cite{foster2020beyond}.
Thus we get
  \begin{align*} \sum_{t=1}^\tau  [\sum_{a'=1}^{K} p_{t,a'}^* \ell_t(a') -  \sum_{a=1}^K p_{t,a} \ell_t(a) ] &\leq  \sum_{t=1}^\tau  \big(\dfrac{2K}{\gamma} + \sum_{a=1}^K\dfrac{\gamma}{4}p_{t,a}(\hat{\ell}_t(a) - \ell_t(a))^2 \big)  .
\end{align*}


Now we would control $\sum_{t=1}^\tau  \sum_{a=1}^K\dfrac{\gamma}{4}p_{t,a}(\hat{\ell}_t(a) - \ell_t(a))^2 :$ To deal with the stopping time $\tau,$ we establish a concentration result of $\sum_{s=1}^{t}\sum_{a=1}^K\dfrac{\gamma}{4}p_{s,a}(\hat{\ell}_s(a) - \ell_s(a))^2 \big)$ uniformly for all $1\leq t_0\leq T$ via the following variant of Freedman's inequality: 

\begin{lemma}[\cite{rakhlin2012making}, Lemma~3]\label{lem-Freedman}  Let $\{D_s\}_{s = 1}^T$ be a martingale difference sequence w.r.t. filtration $\{\mathscr{F}\}_{s = 1}^T$ and with a uniform upper bound $b$. Let $V$ denote the sum of conditional variances, \begin{align*}
    V_s = \sum_{i=1}^s\text{Var}(D_i \lvert \mathscr{F}_{i-1}) ,
\end{align*}
then for any $\delta < 1/e$ and $T\geq 4$, \begin{align*}
    \Probability\bigg(\lvert  \sum_{s=1}^t D_s \rvert  > 2\max\big\{2\sqrt{V_s},b\sqrt{\log(1/\delta)} \big\}\sqrt{\log(1/\delta)} \quad \text{for some } t\leq T \bigg)\leq 2\delta \log(T) .
\end{align*}
\end{lemma}

Now noticing that for $\mathscr{F}_{t-1}: = \sigma(\{x_t\} \cup  \{x_{s},a_{s},r_{s},\bm{c}_i\}_{s \leq t-1} )$,
denote $M_s: = (\hat{\ell}_s(a_s) - \ell_s(a_s))^2$,
we have $$D_s := M_s - \E[M_s\lvert \mathscr{F}_{s-1}] = (\hat{\ell}_s(a_s) - \ell_s(a_s))^2 - \sum_{a = 1}^K p_{s,a}\big (\hat{\ell}_s(a) - \ell_s(a))^2  , 1\leq s \leq T$$ is a martingale difference sequence uniformly bounded by $4Z^2 +4$ with respect to $\mathscr{F}_{s},$ i.e. $\lvert D_s\rvert \leq 4Z^2+4$, $\{D_s\}_{s = 1}^T$ is adaptive to $\{\mathscr{F}_s\}_{s = 1}^T$ and $\E[D_{s}\lvert \mathscr{F}_{s-1}] = 0.$ By \begin{align*}
    \text{Var}(D_i \lvert \mathscr{F}_{i-1} ) \leq \E[ D_i^2 \lvert \mathscr{F}_{i-1}] &\leq (4Z^2+4)\E[ \lvert D_i \rvert \lvert \mathscr{F}_{i-1}] \leq (8Z^2+8)\E[ M_i \lvert \mathscr{F}_{i-1}].
\end{align*}
We have by Lemma~\ref{lem-Freedman}, with probability at least $1- O(1/T^2)$, {\footnotesize \begin{align*}
 \sum_{s = 1}^t \sum_{a = 1}^K p_{s,a} \big( \hat{\ell}_s(a) - \ell_s(a) \big)^2 \leq &  8\sqrt{\log T} \max\bigg\{ \sqrt{(8Z^2+8)\sum_{i = 1}^s\sum_{a = 1}^Kp_{s,a}(\hat{\ell}_{s}(a) - \ell_s(a))^2 },(4Z^2+4)\sqrt{\log T}\bigg\}\\
 & + \sum_{s = 1}^t \big(\hat{\ell}_s(a_s) - \ell_s(a_s) \big)^2  
\end{align*}
}
holds for uniformly for $1\leq t\leq T$.

When $\sum_{s = 1}^t \sum_{a = 1}^K p_{s,a} \big( \hat{\ell}_s(a) - \ell_s(a) \big)^2 > (2Z^2+2) \log T$, we have by \begin{align*}
    A \leq B+ C\sqrt{A} \implies (\sqrt{A} - \frac{C}{2})^2 \leq B + \frac{C^2}{4}\implies A \leq 2{B}+C^2,\quad \forall A,B,C>0,
\end{align*}
where $$A =\sum_{s = 1}^t \sum_{a = 1}^K p_{s,a} \big( \hat{\ell}_s(a) - \ell_s(a) \big)^2,\quad B = \sum_{s = 1}^t \big(\hat{\ell}_s(a_s)- \ell_s(a_s)\big)^2 , \quad C = 8\sqrt{ (8Z^2+8) \log T}  $$
then implies with probability at least $1-O(1/T^2),$
\begin{align*}
    \sum_{s = 1}^t \sum_{a = 1}^K p_{s,a} \big( \hat{\ell}_s(a) - \ell_s(a) \big)^2 &\leq  \max\{(2Z^2+2)\log T,2\sum_{s = 1}^t\big(\hat{\ell}_s(a_s) - \ell_s(a_s) \big)^2  + (512Z^2+512)\log T \} \\
    & = 2\sum_{s = 1}^t \big(\hat{\ell}_s(a_s) - \ell_s(a_s) \big)^2  + (512Z^2+512)\log T, \quad \forall 1\leq t \leq T.
\end{align*}
Finally, by our assumption on $\mathcal{R}^r,\mathcal{R}^c,$ \begin{align*}
    \sum_{s = 1}^t \big(\hat{\ell}_s(a_s) - \ell_s(a_s) \big)^2 &=\sum_{s=1}^t \big(\hat{r}_{s,a_s}  - f^*(x_s,a_s) + \bm\lambda_s^T (\bm g^*(x_s,a_s)-\hat{ \mathbf{c}}_{s,a_s}))^2\\
& \leq \sum_{s=1}^t 2\bigg(  [\hat{r}_{s,a_s} - f^*(x_s,a_s)]^2 +Z^2 \lVert \bm g^*(x_s,a_s)- \hat{ \mathbf{c}}_{s,a_s} \rVert_\infty ^2  \bigg). \\& \leq 2\Regsq^r(T)+ 2 Z^2 \Regsq^c(T).
\end{align*}
we get with probability at least $1-O(1/T^2)$, 
\begin{align*}
 \sum_{t=1}^\tau \sum_{a = 1}^Kp_{t,a} \ell_t(a) & \geq  \sum_{t=1}^\tau\sum_{a=1}^K {p}^*_{t,a} \ell_t(a)   - O( \dfrac{2KT}{\gamma}  +\dfrac{\gamma}{4}\big[  \text{Reg}_{sq}^r(T)+Z^2 \text{Reg}_{sq}^c(T)+(Z^2+1)\log T \big] ). \end{align*}
 Applying Azuma-Hoeffding inequality to $\sum_{t = 1}^T \ell_t(a_t),$ and by our selection of $\gamma$,  we get with probability at least $1-O(1/T^2),$ \begin{align*}
	\sum_{t=1}^{\tau}\ell_t(a_t)
	 \geq  &\sum_{t=1}^\tau\sum_{a=1}^K { p}^*_{t,a} \ell_t(a)   - O((Z+1) \sqrt{KT [\text{Reg}_{sq}^r (T)+\text{Reg}_{sq}^c (T)] }) \\
	 =&\sum_{t=1}^\tau \bigg[ \langle \sum_{a = 1}^K p^*_{t,a}{f}^{*}(x_t,a)+\bm \lambda_t^T {\big (\dfrac{B}{T}  \cdot \bm 1 - \sum_{a= 1}^K {p}_{t,a}^* \bm g^*(x_t,a)  \big)}\bigg]    \\& - O( (Z+1) \sqrt{KT [\text{Reg}_{sq}^r (T)+\text{Reg}_{sq}^c (T)+\log T] }) . 
	 \end{align*} 
Now consider a new filtration $\mathscr{F}_{t-1}' = \sigma(\{x_s,a_s,r_s,\bm{c}_s\}_{s = 1}^{t-1})$, we have denote 
$$D_t': = \sum_{a = 1}^K p^*_{t,a}{f}^{*}(x_t,a)+\bm \lambda_t^T \big(\dfrac{B}{T}  \cdot \bm 1 - \sum_{a= 1}^K {p}_{t,a}^* \bm g^*(x_t,a)  \big) - \OPT,$$
then $D_t'$ is adaptive to $\mathscr{F}^{'}_{t}$, $\lvert D_t'\rvert \leq Z +2$, and $\E[D_t'\lvert \mathscr{F}_{t-1}']\geq 0 $. By the Azuma-Hoeffding inequality, we get with probability at least $1-O(1/T^2),$ \begin{align*}
\sum_{s = 1}^t D_s' \gtrsim -(Z+1)\sqrt{T\log T}, \quad \forall 1\leq t\leq T.
\end{align*}
That leads to with probability at least $1-O(1/T^2),$ \begin{align*}
   \sum_{t=1}^{\tau}\ell_t(a_t) \geq \tau \OPT  - O((Z+1) \sqrt{KT[ \text{Reg}_{sq}^r (T)+\text{Reg}_{sq}^c (T) + \log T ]}).
\end{align*}
On the other hand, by Lemma~\ref{lem-OCO}, for any fixed $\bm \lambda \in \Lambda$ we have, \begin{align*}
	\sum_{t=1}^\tau f^*(x_t,a_t) + \bm\lambda_t^T\big (B/T\cdot \bm 1 - \bm c_{t,a_t}   \big )  &\leq  \sum_{t=1}^\tau 	 f^*(x_t,a_t) + \bm\lambda ^T(B/T\cdot \bm 1 - \bm c_{t,a_t}  )  \big)+ O(Z \sqrt{T\log d}).
	\end{align*}
 By applying Azuma-Hoeffding inequality to summation of $\bm\lambda_t^T(B/T \bm 1 - \bm{c}_{t,a_t})$ with respect to $\tilde{\mathscr{F}}_{t-1}: = \sigma\{\mathscr{F}_t \cup \{a_t\}\}$, we get with probability at least $1-O(1/T^2),$ \begin{align*}
	\sum_{s=1}^t\bm\lambda^T_s(B/T\cdot \bm 1 - \bm c_{s,a_s} ) &\geq 	\sum_{s=1}^t  \bm\lambda_s^T(B/T \cdot \bm 1  - \bm  g^* (x_s,a_s)) - O(Z+1)\sqrt{T\log(T)}), \quad \forall 1\leq t\leq T.
\end{align*}
Combining all results together, we get with probability at least $ 1 - O(1/T^2),$ {\small  \begin{align*}
	 \sum_{t=1}^\tau f^*(x_t,a_t) + \bm \lambda^T(B/T\cdot \bm 1 - \bm c_{t,a_t}) \geq  \tau \cdot\text{OPT}    - O\big ( (Z+1)\sqrt{KT [\text{Reg}_{sq}^r (T)+\text{Reg}_{sq}^c (T)+\log (dT)] } \big ), \quad \forall \bm\lambda\in \Lambda.
\end{align*}}
Now\\
\textbf{Case1:} if $\tau = T,$ we get the desired regret bound by selecting $\bm \lambda = \bm 0$.\\
\textbf{Case2:} If $\tau < T,$ there exists some resource $j$ running out, i.e. $\sum_{t=1}^\tau (c_{t,a_t})_j > B-1,$  then letting $\bm \lambda = Z  \bm e_j,$ we get \begin{align*}
	 \sum_{t=1}^\tau\big[  f^*(x_t,a_t) + \bm \lambda^T(B/T\cdot \bm 1 - \bm c{t,a_t}) \big] & < \sum_{t=1}^\tau f^*(x_t,a_t)  + Z\big( \tau B/T - B+1  \big)  \\
 &\leq  \sum_{t=1}^\tau f^*(x_t,a_t)  + \OPT\big( \tau  - T  \big)+Z.
\end{align*}
where the second inequality is by $\tau /T <1$ and $Z \geq T \OPT/B$. This inequality leads to with probability at least $1-O(1/T^2)$ \begin{align*}
    \sum_{t = 1}^\tau f^*(x_t,a_t) > T\OPT - Z - O\big ( (Z+1)\sqrt{KT [\text{Reg}_{sq}^r (T)+\text{Reg}_{sq}^c (T)+\log (dT)] } \big ).
\end{align*}
Denote the event that above inequality holds as $\mathcal{Q},$ then \begin{align*}
    \Reg(T) &\leq  \E[\big (T \OPT -  \sum_{t = 1}^{\tau} f^*(x_t,a_t) \big) \big( \bm{1}_{Q}+ \bm{1}_{Q^c}\big)]  \\
    &\leq \E[ \big (T \OPT -  \sum_{t = 1}^{\tau} f^*(x_t,a_t) \big) \bm{1}_{\mathcal{Q}}]+ T\Probability(Q^c)\\
    &\lesssim  (Z+1)\sqrt{KT [\text{Reg}_{sq}^r (T)+\text{Reg}_{sq}^c (T)+\log (dT)] }.
\end{align*}
Thus the claim holds.
\end{proof}

\subsection{Proof of Theorem~\ref{thm-lb-general}}

\begin{proof}
When $\Reg_{\text{CB}}(T;\mathcal{F}) > \alpha T,$ one can set $B = T$. Then the problem is exactly the unconstrained CB problem with reward class $\mathcal{F}$, whose regret is lower bounded by $\Reg_{\text{CB}}(T;\mathcal{F}).$ 

Now we would focus on the case $\alpha T > \Reg_{\text{CB}}(T;\mathcal{F}):$ Firstly assume W.L.O.G. $\alpha T$ is not $O(1),$ otherwise the lower bound result is trivial.

For any fixed $\pi$, consider the following instance of CBwK:
\begin{enumerate}
    \item At every round $t$, $x_t$ is sampled i.i.d. from $P_0$.
    \item By condition of Theorem~\ref{thm-lb-general}, there exists some $f_0 \in \mathcal{F}$ satisfying $f_0(x,0) = f_0(x,1) \geq  c_0$ a.s. $P_0$. We set the reward as $r_{t,a} = f_0(x_t,a)$. 
    \item By $\mathcal{G}$ is $\alpha$-separable, there exists some $g_0 \in \mathcal{G}$ and distribution $Q_0$ of costs such that \begin{align*}
        p: =\E_{P_0}[\min_{a}g_0(x,a)] \geq  \frac{1}{4}, \quad \E_{P_0}[\max_{a}g_0(x,a)]\geq p+\alpha.
    \end{align*}
    and there exists some $c_1>0$ s.t. \begin{align*}
        \E_{P_0,Q_0}[\sum_{t = 1}^{T/8} \bm{1}\{ \pi_t(x_t) \neq \text{argmin}_{a}g_0(x,a) ]> c_1T.
    \end{align*}
    we let the cost be generated from $Q_0$.
\end{enumerate}

For such instance, we have the regret of $\pi$ is lower bounded by 
$c_0 (\E[\tau^*] - \E[\tau^\pi]),$ with $\tau^\pi$ the stopping time of $\pi$ and $\tau^*$ the stopping time of $\pi^*,$ with \begin{align*}
    \pi^*_t(x_t) = \text{argmax}_{a} g_0(x_t,a).
\end{align*}
Now we would bound $\E[\tau^* - \tau^\pi]$ from below: 

\paragraph{Lower bound of $\E[\tau^*]$:}
For $U_i$ i.i.d. from the distribution of $c_{1,\pi^*(1,x_1)}$, we have then $\tau^*$ follows the same distribution as $\tau_U\wedge T$. 
Consider the first-hitting times $$\inf\{t\geq 0: \sum_{s = 1} ^t U_s >B-1\},\quad  \tau_V:= \inf\{t\geq 0: \sum_{s = 1} ^t V_s >B-1\},$$
notice that by $\E[U_i] = p\geq 1/4$ and $U_i \in[0,1]$, we have $\tau_U<\infty$ almost surely.  Thus by  Wald's equation, \begin{align*}
    B-1< \E[\sum_{t = 1}^\tau U_t ] = \E[U_1] \E[\tau_U],
\end{align*}
that leads to $\E[\tau_U] \geq \frac{B-1}{p}. $
Notice that for $\tau_U' = \tau_U \wedge T,$ we have \begin{align*}
     \E[\tau_U] - \E[\tau_U']  & = \sum_{t \geq 0} \Probability(\tau_U > 0) - \sum_{t \geq 0} \Probability(\tau_U' > 0) \\
     & = \sum_{t\geq T} \Probability(\tau_U > t)\\
     & \leq \sum_{t\geq T} \Probability( \sum_{s = 1}^m U_s < B-1, \forall m \leq t )\\
     & \leq \sum_{t\geq T} \Probability( \sum_{s = 1}^{t} U_s < B-1).
\end{align*}
Notice that by Hoeffding's inequality and $p\geq 1/4$,  \begin{align*}
     \Probability(  \sum_{s = 1}^t  U_s   < \frac{t}{4}-u )  \leq \Probability(\lvert  \sum_{s = 1}^t  U_s -  pt\rvert    > u ) \lesssim \exp( -C u^2/t),
\end{align*}
thus by $B = T/8,$  \begin{align*}
    \Probability( \sum_{s = 1}^{t} U_s < B-1)  \leq   \Probability( \sum_{s = 1}^{t} U_s < \frac{t}{4} -\frac{t}{8}) \lesssim \exp(-ct), \quad \forall t > T.
\end{align*}
That leads to \begin{align*}
     \E[\tau_U] - \E[\tau_U']  \lesssim \sum_{t \geq T} \exp(-ct)\lesssim \exp(-cT).
\end{align*}
Thus we get $$\E[\tau^*] = \E[\tau_U'] \geq \frac{B-1}{p} - O\big (\exp(-cT)\big ).$$
     
\paragraph{Upper bound of $\E[\tau^\pi]$:} Recall the notation $\mathscr{F}_{t-1}: = \sigma(\{x_t\}\cup \{x_{s},a_s,r_s,c_s\}_{s = 1}^{t-1})$, since every pulling with $\pi_t(x_t) \neq \text{argmin}_{a} g_0(x_t,a)$ will incur a cost at least $p+\alpha$ in expectation, we have $$\E[c_{t,\pi(x_t)}\lvert {\mathscr{F}}_{t-1}]\geq p+\alpha - \alpha\E[\bm{1}\{\pi(x_t) = \pi^*_{\epsilon_0}(x_t)\}\lvert \tilde{\mathscr{F}}_{t-1}].$$ 
Thus if we denote $$\tilde{D}_t: = c_{t,\pi(x_t)} -(p+\alpha)  + \alpha \bm{1}\{\pi(x_t) = \pi^*(x_t)\}, t \geq 1, \quad \tilde{D}_0 = 0.  $$
Then $\tilde{D}_t$ is a sub-martingale difference sequence with respect to ${\mathscr{F}}_{t}$, i.e.  $\tilde{M}_t:=\sum_{s = 0}^t \tilde{D}_s$ is a sub-martingale with respect to $\mathscr{F}_t$. Denote $\tilde{\tau}': = \inf\{t\geq 0, \sum_{t} c_{t,\pi(x_t)}> B-1 \},$ 
 then $\tilde{\tau}'<\infty$ a.s., thus we have by optional stopping theorem, 
   $\E[\tilde{M}_{\tilde{\tau}'}] \geq 0$.
   
   If we denote $C_t$ the total cost of $\pi$ up to time $t$ and $N(t)$ the total times of pulling on or before $t$ so that $\pi_t\neq \pi^*(x_t)$, then \begin{align}
    \E[\tilde{M}_{\tilde{\tau}'}] = \E[C_{\tilde{\tau}'} - {(p + \alpha)\tilde{\tau}'}+ \alpha N(\tilde{\tau}')] \geq 0,
   \end{align}
thus by $\tilde{\tau}'\geq T/8$ \begin{align*}
        B \geq \E[(p+\alpha) \tilde{\tau}' -\alpha N(T/8)] \geq (p+\alpha)\E[\tilde{\tau}'] - c_1\alpha T ,
   \end{align*}
i.e. $\E[\tilde{\tau}'] \leq  \frac{B}{p+\alpha} - \frac{c_1\alpha}{p+\alpha}T.  $

Now since $\tau^\pi = \tilde{\tau}' \wedge T,$ we have $$\E[\tau^\pi] \leq \E[\tilde{\tau}'] \leq  \frac{B}{p+\alpha} - \frac{c_1\alpha}{p+\alpha}T.$$

\paragraph{Lower bound of $\E[\tau^*-\tau^\pi]$} Combing bounds for $\E[\tau^\pi ], \E[\tau^*]$ together, we have \begin{align*}
    \E[\tau^*-\tau^\pi] &\geq  (\frac{1-1/B}{p} - \frac{1}{p+\alpha}) B + \frac{c_1\alpha T}{p+\alpha}  - O(\exp(-cT)\\
    &\gtrsim \alpha T + ( p+\alpha - \frac{p+\alpha}{B} -p)B-O(\exp(-cT))\\
    &\gtrsim \alpha T - O(1).
\end{align*} 
Noticing that $\alpha T$ dominates the last term, thus the claim holds.
\end{proof}

\subsection{Detail of Results in Section~\ref{sec-example}}\label{sec-appendix-example-proof}

\subsubsection{Construction of $\mathcal{R}^c$ from $\tilde{\mathcal{R}}^c$}
By $\mathcal{G} = \tilde{\mathcal{G}}^d,$ we can denote the underlying expected cost function $\bm g^*(x,a) =\big( g^*_1(x,a),\dots, g^*_d(x,a)\big).$ By our assumption on $\tilde{R}^c$, we have running $\tilde{R}^c$ over a online regression problem with underlying function $g_i^*(\cdot,\cdot)$ will generate a sequence of predictors $\{\hat{g}_{t,i}(\cdot,\cdot)\}_{t = 1}^T$ so that \begin{align*}
    \sum_{t = 1}^T \big( \hat{g}_{t,i}(x_t,a_t) - g^*_i(x_t,a_t) \big)^2 \leq \widetilde{\Reg}_{\text{sq}}^c(T).
\end{align*} 
So if we set $\mathcal{R}^c$ as the oracle with output $\hat{\bm g}_t(x,a) : = \big( \hat{g}_{t,1}(x,a),\dots, \hat{g}_{t,d}(x,a) \big)$ at each round $t$, then it satisfies that \begin{align*}
    \sum_{t = 1}^T  \lVert \hat{\bm g}_t(x_t,a_t) - {\bm g}_t^*(x_t,a_t) \rVert_\infty^2 \leq \sum_{i = 1}^d \sum_{t = 1}^T \big( \hat{g}_{t,i}(x_t,a_t) - g^*_i(x_t,a_t) \big)^2  \leq  d \widetilde{\Reg}_{\text{sq}}^c(T).
\end{align*}
\subsubsection{Generalized Linear CBwK}

\paragraph{Selection of Oracles and Upper Regret Bound} For the $m$-dimensional generalized linear class, by Proposition~3.3 of \cite{foster2020beyond}, the Newtonized GLMtron oracle achieves the $O(m\log T) $ online regression regret. So selecting $\mathcal{R}^r,\tilde{\mathcal{R}}^c$ as GLMtron oracle implies \begin{align*}
    \Regsq^r(T) \lesssim m_1\log T, \quad    \Regsq^c(T) \lesssim m_2 d\log T.
\end{align*}
Bringing this result to Theorem~\ref{thm-regret-upper-bound} leads to the desired regret upper bound of generalized linear CBwK.

\paragraph{Regret Lower Bound Result} 
Since Generalized linear CBwK includes linear CBwK as a special case, we focus on proving the lower bound result in linear setting.
Formally, we show the following lower bound result for linear CBwK: 
\begin{theorem}\label{thm-lb-linear}
    Let $\mathcal{F},\mathcal{G}$ be $m_1$-dimensional and $m_2$-dimensional linear function classes respectively. Then there exists the selection of feature maps $\phi_1,\phi_2$, $T>0, B\geq T/4$, so that for a CBwK instance with time horizon $T$, budget $B$ and the reward and cost function classes $\mathcal{F},\mathcal{G}$ with $K = 2, d = 1$, any policy $\pi$ must have \begin{align*}
        \E[\Reg_{\pi}(T) ]  = \tilde{\Omega}( \sqrt{(m_1+m_2)T} )
    \end{align*}
\end{theorem}

\begin{proof}[Proof of Theorem~\ref{thm-lb-linear}]
We would prove Theorem~\ref{thm-lb-non-parametric} by verifying the conditions in Theorem~\ref{thm-lb-general}. Our analysis includes three steps:
\begin{enumerate}
    \item Specify $\mathcal{X}$ and feature maps $\phi_1,\phi_2.$
    \item Verify $\mathcal{G}$ is $\alpha$-inseparable.
    \item Verify the existence of $f_0\in \mathcal{F}$ which satisfies the condition of Theorem~\ref{thm-lb-general}.
    \item Plug the $\alpha T$ and $\Reg_{\text{CB}}(T;\mathcal{{F}})$ into Theorem~\ref{thm-lb-general} to get the desired lower bound.
\end{enumerate}

\textbf{Step1: } We specify $\mathcal{X}$ as the a subset $[0,1]^{2m_1+2m_2}$ defined as following: \begin{align*}
    \mathcal{X} = \{x = (\underbrace{x^{(1)}}_{\mathbb{R}^{m_1}},\underbrace{x^{(2)}}_{\mathbb{R}^{m_1}},\underbrace{x^{(3)}}_{\mathbb{R}^{m_2}},\underbrace{x^{(4)}}_{\mathbb{R}^{m_2}}): \lVert x^{(i)}\rVert_2\leq 1, i = 1,2,3,4 \}.
\end{align*}
The feature maps $\phi_1(x,a),\phi_2(x,a)$ is defined as \begin{align*}
    \phi_1(x,0) = x^{(1)}, 
    \phi_1(x,1) = x^{(2)},
    \phi_2(x,0) = x^{(3)},
    \phi_2(x,1) = x^{(4)}.
\end{align*}
With such construction, determine the distribution over $x$ is equivalent to determine the distribution of $ (x^{(1)},x^{(2)})$ and $(x^{(3)},x^{(4)})$, we will determine them in Step2 and Step3.

\noindent \textbf{Step~2: } As constructed in Step~1, we need only construct the distribution of $(\phi_2(x,0),\phi_2(x,1))$. Our construction  is motivated by a deterministic construction used in \cite{chu2011contextual}.
Consider the following subset of $[0,1]^{2m_2}:$ 
W.L.O.G. assume $r = (m_2-1)/2$ is an integer, let \begin{align*}
    \tilde{\mathcal{X}}: = \{ \tilde{x}_i:= \frac{1}{2}\bm{e}_1 + \frac{1}{2}  \bm{e}_{2i} + \frac{1}{2} \bm{e}_{m_2+1} + \frac{1}{2} \bm{e}_{m_2+2i+1}, i \in[r]\} ,
\end{align*}
then when $(x^{(3)},x^{(4)})\in \tilde{\mathcal{X}}$, we have \begin{align*}
    \phi_2(x,0) = (\frac{1}{2},0,\dots,{\frac{1}{2},0},\dots,0), 
    \phi_2(x,1) = (\frac{1}{2},0,\dots,0,\frac{1}{2},\dots,0).
\end{align*}
We let the distribution of $(x^{(3)},x^{(4)})$ $\tilde{P}_0$ as the uniform distribution over $\tilde{\mathcal{X}}$.

On the other hand, we construct the subset $\mathcal{U}_\gamma$ of $\mathcal{G}$ as following: For every $\epsilon\in \{-1,1\}^r,$ let \begin{align*}
  (\theta_{\epsilon})_1 = \frac{1}{2}, ( \theta_{\epsilon})_{2i,2i+1} = \left\{\begin{matrix}
       (\gamma,0)   & \text{if } \epsilon_i = -1, \\
       (0,\gamma)   & \text{if } \epsilon_i = 1, \quad i \in[r], \quad (\theta_{\epsilon})_j = 0 \text{ for other }j.
   \end{matrix}\right. 
\end{align*}
we set \begin{align*}
    \mathcal{U}_\gamma: = \{ (x,a) \to \langle \phi_2(x,a), \theta_\epsilon\rangle ,\quad \epsilon\in \{-1,1\}^{r} \},
\end{align*}
then for any $g\in \mathcal{U}_\gamma,$ we have \begin{align*}
  \frac{1}{4} =  \E_{x}[\min_{a}\phi_2(x,a)]<\E_{x}[\max_{a}\phi_2(x,a)] = \frac{1}{4}+\frac{\gamma}{2}
\end{align*}
If we set $r(x,a)$ generated as $\text{Bernoulli}(\phi_2(x,a))$ condition on $x,a$ and consider the uniform distribution  $Q$  over CB instances with underlying $g^*\in \mathcal{U}_{\gamma}$, since by our construction $g(\tilde{x}_i,a), g(\tilde{x}_j,a)$ are independently distributed under $Q$ when $\tilde{x}_i\neq \tilde{x}_j$, we have when consider the CB instance with $g\sim Q$, the data generating process can be seen as following, as discussed in \cite{foster2020beyond}: \begin{enumerate}
    \item Sample $x_1,\dots,x_T$ i.i.d. from $\tilde{P}_0$, set $S_i:=\{t\in [T]: x_t = \tilde{x}_i \}$.
    \item For each $i \in [r]$, independently sample a Bernoulli MAB instance $\mathcal{P}_i$ with arm means $\mu_1,\mu_2$ such that with probability $\frac{1}{2}$, \begin{align*}
        \mu_1 = \frac{1}{4}, \mu_2 = \frac{1}{4}+\gamma
    \end{align*}
    and with probability $\frac{1}{2},$ \begin{align*}
        \mu_2 = \frac{1}{4}+\gamma, \mu_1 = \frac{1}{4}.
    \end{align*}
\end{enumerate}
Now by the same statement as in \cite{foster2020beyond}, we have selecting $\gamma \propto \sqrt{\frac{r}{T}} \propto \sqrt{\frac{m_2}{T}} $ leads to for any $\pi,$ there exists some $\epsilon_0\in\{-1,1\}^r$ so that for the CB instance with underlying reward function is $g_{\epsilon_0},$ \begin{align*}
   \frac{\gamma}{2}\E[\sum_{t = 1}^T  \bm{1}\{ \text{argmax}_{a} g_{\epsilon_0}(x,a) \neq \pi_{t}(x_t) \} ]    = \E[\sum_{t = 1}^T \max_{a}g_{\epsilon_0}(x_t,a ) - g_{\epsilon_0}(x_t,\pi_t(x_t))] \gtrsim \gamma T.
\end{align*}
Thus $\mathcal{G}$ is $\sqrt{\frac{m_2}{T}}$-inseparable. Above result also implies a $\Omega(\sqrt{m_2T})$  lower bound for $\Reg_{\text{CB}}$ over $\mathcal{G}$.\\

\noindent \textbf{Step2:} We can simply set the distribution $\tilde{P}'$ of $(x^{(1)},x^{(2)})$ as a constant distribution  $x_0$ with $\frac{1}{4}$ in its first and $m_1+1$-th coordinate and $0$ in other coordinates. Then let $P_0$ be the distribution of $(x^{(1)},x^{(2)},x^{(3)},x^{(4)})$ given by $(x^{(1)},x^{(2)}) \sim \tilde{P}', (x^{(3)},x^{(4)}) \sim \tilde{P}$, we have $f_0(x,a) = \langle \phi_1(x,a), \theta_0)$ with $\theta_0 = x_0$  is a function satisfying the condition of Theorem~\ref{thm-lb-general}. 

\noindent \textbf{Step3:}  By our construction, we have there exists some $f_0 \in \mathcal{F}$ so that $f(x,a) \equiv \frac{1}{4}$ under $P_0$, $\text{Reg}_{\text{CB}}(T;\mathcal{F})= \tilde{\Omega}( \sqrt{m_1T}).$ And $\mathcal{G}$ is $\tilde{\Theta}(\sqrt{\frac{m_2}{T}})$-inseparable, applying Theorem~\ref{thm-lb-general} leads to the $\tilde{\Omega}(\max\{\sqrt{m_1T},\sqrt{m_2 T} \})=\tilde{\Omega}(\sqrt{m_1T}+\sqrt{m_2 T}  )$ lower bound as desired.
\end{proof}

\subsubsection{Nonparametric CBwK}

\paragraph{Selection of Oracles and Upper Regret Bound} For the $p$-nonparametric class, by Theorem~3 of \cite{foster2020beyond}, a Vovk's aggregation based oracle achieves the $O(KT)^{1-\frac{2}{2+p}}$ online regression regret. So selecting $\mathcal{R}^r,\tilde{\mathcal{R}}^c$ as this oracle implies \begin{align*}
    \Regsq^r(T) \lesssim (KT)^{1-\frac{2}{2+p_1}}, \quad    \Regsq^c(T) \lesssim d(KT)^{1-\frac{2}{2+p_2}}.
\end{align*}
Bringing this result to Theorem~\ref{thm-regret-upper-bound} leads to the desired regret upper bound of nonparametric CBwK.

\paragraph{Regret Lower Bound Result}

Formally, we show the following lower bound result for non-parametric CBwK: 
\begin{theorem}\label{thm-lb-non-parametric}
    Let $\mathcal{W},\mathcal{V}$ be two function classes consisting of functions from $\mathcal{X}$ to $[0,1]$, satisfying $H(\mathcal{W})  = \Theta(\varepsilon^{-p_1}), H(\mathcal{V})  = \Theta(\varepsilon^{-p_2}).$ Then there exists $T>0, B\geq T/4$, a slightly modified class $\mathcal{V}',\mathcal{W}'$ with  $H(\mathcal{W}')  = \tilde{\Theta}(\varepsilon^{-p_1}), H(\mathcal{V}')  = \tilde{\Theta}(\varepsilon^{-p_2})$ so that for a CBwK instance with time horizon $T$, budget $B$ and the reward and cost function classes $\mathcal{F},\mathcal{G}$ constructed from $\mathcal{V}',\mathcal{W}'$ with $K = 2, d = 1$, any policy $\pi$ must have \begin{align*}
        \E[\Reg_{\pi} ]  = \tilde{\Omega}(T^{\frac{1+p_1}{2+p_1}}+ T^{\frac{1+p_2}{2+p_2}} )
    \end{align*}
\end{theorem}

Theorem~\ref{thm-lb-non-parametric} can be seen as a extension of Theorem~2 in \cite{foster2020beyond} for non-parametric contextual bandits. One subtlety of these nonparametric results is that, in contrast to the results in linear case, the lower bound is stated over a modification of the original function classes $\mathcal{F},\mathcal{G}$. That makes the result slightly different from the classical worst-case lower bound results for CB or CBwK. However, such lower bound implies the optimality of our algorithm with respect to the complexity of the considered function classes, which receive the most attention and often provide tight characterization of the problem in nonparametric setting.

\begin{proof}[Proof of Theorem~\ref{thm-lb-non-parametric}]
Similarly to the proof of Theorem~\ref{thm-lb-linear}, we would use Theorem~\ref{thm-lb-general} to develop Theorem~\ref{thm-lb-non-parametric}. Our analysis includes three steps: \begin{enumerate}
    \item Construct a modification $\mathcal{V}'$ of $\mathcal{V}$ so that its corresponding cost function class $\mathcal{G}$ satisfies the $\alpha$-inseparable property.
    \item Construct a modification $\mathcal{W}'$ of $\mathcal{W}$ so that its corresponding cost function class $\mathcal{F}$ containing some $f_0\in \mathcal{F}$ that satisfies the condition of Theorem~\ref{thm-lb-general}.
    \item Plug the $\alpha T$ and $\Reg_{\text{CB}}(T;\mathcal{{F}})$ into Theorem~\ref{thm-lb-general} to get the desired lower bound.
\end{enumerate}

\textbf{Step1: }
The construction of $\mathcal{G}$ in Step~1 is from the construction used in the proof of Theorem~2 in \cite{foster2020beyond}. And we restate several key steps for completeness. \\
For any $0<\gamma\leq \frac{1}{4}$ and $\mathcal{V}$ with $H(\mathcal{V}) = \Theta(\varepsilon^{-p_2}),$ by the argument in \cite{foster2020beyond}, we can find $m = \Theta(\gamma^{-p_2})$ distinct $x^{(1)},\dots,x^{(m)}\in \mathcal{X}$  and a class $\mathcal{V'}\supset \mathcal{V} $ containing $\{h_\epsilon\}_{\epsilon \in \{-1,1\}^m}$ such that \begin{align*}
	h_\epsilon:\mathcal X \to [0,1],\quad  h_\epsilon(x^{(i)}) = \dfrac{1+\epsilon_i\gamma}{2}.
\end{align*}

For the cost class $\mathcal{G}$ corresponds to $\mathcal{V}'$, we consider its subset $\mathcal{U}_\gamma$:
\begin{align*}
    \mathcal{U}_\gamma = \{g_{\epsilon}(x,a): g_{\epsilon}(x,0) = h_\epsilon(x),\quad g_{\epsilon}(x,1) = \frac{1}{2}, \epsilon\in \{-1,1\}^m  \}.
\end{align*}
Let $P^\gamma$ be the uniform distribution over $\{x^{(1)},\dots,x^{(m)}\}$, then for
\begin{align*}
	\pi_{\epsilon}^*(x^{(i)}) = \left\{\begin{matrix}
	 	 0, & \text{ if }\epsilon_i = -1, \\
		1, & \text{otherwise}.
	\end{matrix}\right. ,
\end{align*}
we have
\begin{align*}
\max_{a}g_\epsilon(x,a)- \min_{a}g_\epsilon(x,a)> \frac{\gamma}{2}, \quad  \frac{1}{4}\leq \frac{1-\gamma}{2} \leq  \E_{x}[\min_{a} g_{\epsilon}(x,a)]<
   \E_{x}[\max_{a} g_{\epsilon}(x,a)]\leq  \frac{1+\gamma}{2}\leq \frac{3}{4}.
\end{align*}
Now for every sufficiently large $T$, the argument in \cite{foster2020beyond} 
shows that when letting $\gamma \propto T^{-\frac{1}{2+p_2}}/\text{polylog}(T)$, $P_{\mathcal{X}} = P_{\gamma}$, for every every policy $\pi$, there exists some $\epsilon_0$ so that when the reward is generated as $r(x_t,a)\sim \text{Bernoulli}(g_{\epsilon_0}(x,a)), $ it holds that \begin{align*}
   \frac{\gamma}{2}\E[\sum_{t = 1}^T  \bm{1}\{ \pi_{\epsilon_0}^*(x_t)\neq \pi_{t}(x_t) \} ]    = \E[\sum_{t = 1}^T g_{\epsilon_0}(x_t,\pi_{\epsilon_0}^*(x_t) ) - g_{\epsilon_0}(x_t,\pi_t(x_t))] \gtrsim \gamma T.
\end{align*} 
This implies the constructed $\mathcal{G}$ is $\tilde{\Theta}(T^{-\frac{1}{2+p_2}})$-inseparable.

\noindent \textbf{Step2:} Using the same proof as Theorem~2 in \cite{foster2020beyond}, for $\mathcal{W}$ with $H(\mathcal{W},\varepsilon) = \Theta (\varepsilon^{-p_1}), $ we can construct some $\mathcal{W}'$ so that the corresponding $\mathcal{F}'$ satisfies $\Reg_{\text{CB}}(T;\mathcal{F}) = \tilde{\Omega}( T^{\frac{1+p_1}{2+p_1} })$, we just construct $\mathcal{F}$ by adding a constant function $f(x,a)\equiv 1$ into $\mathcal{F}'$. 

\noindent \textbf{Step3:} By our construction, we have there exists some $f_0 \in \mathcal{F}$ so that $f(x,a) \equiv 1$, $\text{Reg}_{\text{CB}}(T;\mathcal{F})= \tilde{\Omega}( T^{\frac{1+p_1}{2+p_1}}). $ And $\mathcal{G}$ is $\tilde{\Theta}(T^{-\frac{1}{2+p_2}})$-inseparable, applying Theorem~\ref{thm-lb-general} leads to the $\tilde{\Omega}(\max\{T^{\frac{1+p_1}{2+p_1}},T^{\frac{1+p_1}{2+p_2}}  \})=\tilde{\Omega}(T^{\frac{1+p_1}{2+p_1}}+T^{\frac{1+p_1}{2+p_2}}  )$ lower bound as desired.
\end{proof}

\section{PROOF OF RESULTS IN SECTION~\ref{sec-general-B} }
\subsection{Estimation Error of Online-To-Batch Conversion Oracle}\label{sec-OTB}

Given online regression oracles $\mathcal{R}^r,\mathcal{R}^c$ satisfying \eqref{regret-reward-oracle} and \eqref{regret-cost-oracle}, we formally define the OTB oracles $\mathcal{R}^r_{\text{est}},\mathcal{R}^c_{\text{est}}$ 
as following: Given the dataset $\mathcal{D}_a$ as in Assumption~\ref{assumption-relaxB}, if we run online oracles $\mathcal{R}^r,\mathcal{R}^c$ over $\mathcal{D}_a$, and suppose the output of online oracles are given by $\{f_{i}(\cdot,a)\}_{i = 1}^M, \{\bm{g}_{i}(\cdot,a)\}_{i = 1}^M$, the output $\hat{f}(\cdot,a),\hat{\bm{g}}(\cdot,a)$ of $\mathcal{R}^r_{\text{est}},\mathcal{R}^c_{\text{est}}$ is defined as following: \begin{align}\label{eq-OTB-oracle-def}
    \hat{f}(x,a): = \dfrac{1}{M}\sum_{i=1}^M f_i(x,a),\quad  
    \hat{\bm g}(x,a): = \dfrac{1}{M}\sum_{i=1}^M {\bm g}_i(x,a).
\end{align}
Now we would show the following estimation error guarantee of $\hat{f},\hat{\bm g}$ defined above: 
\begin{lemma}\label{lem-OTB-est-error}
    For every $a$, suppose the online regression oracles $\mathcal{R}^r,\mathcal{R}^c$ satisfies \eqref{regret-reward-oracle} and \eqref{regret-cost-oracle}, we have $\hat{f}(\cdot,a),\hat{\bm g}(\cdot,a)$ defined in \eqref{eq-OTB-oracle-def}  satisfies Assumption~\ref{assumption-relaxB} with \begin{align*}
       \mathcal{E}_{\delta}(M;\mathcal{F})\lesssim \frac{\Reg_{sq}^r(M) \log({1}/{\delta})}{M},\quad   \mathcal{E}_{\delta}(M;\mathcal{G})\lesssim \frac{\Reg_{sq}^c(M) \log({1}/{\delta}) }{M}.
    \end{align*}
\end{lemma}

\begin{proof}[Proof of Lemma~\ref{lem-OTB-est-error}]
We would prove the result for $\hat{\bm g},$  the proof of $\hat{f}$ is similar.\\

By the convexity of $x\to x^2,$ we have  \begin{align*}
	&\E_{x}[\lVert \dfrac{1}{M}\sum_{t = 1}^{M} \bm{g}_t(x,a) - \bm{g}^*(x,a)\rVert_\infty^2] \\
	\leq & \dfrac{1}{M}\sum_{t = 1}^{M} \E_{x}[ \lVert  {\bm g}_t(x,a) - {\bm g}^*(x,a)\rVert_\infty^2]\\
\leq &\dfrac{\Reg_{sq}^r(M)}{M}  +\dfrac{1}{M}\sum_{t = 1}^{M} \big(\E_{x}[\lVert  {\bm g}_t(x,a) - {\bm g}^*(x,a)\rVert_\infty^2]- \lVert  {\bm g}_t(x_t,a) - {\bm g}^*(x_t,a)\rVert_\infty^2\big).
\end{align*}

In particular, we have for $W_t: = \underbrace{\lVert  {\bm g}_t(x_t,a) - {\bm g}^*(x_t,a)\rVert_\infty^2}_{:= V_t} - \E_{x}[\lVert  {\bm g}_t(x,a) - {\bm g}^*(x,a)\rVert_\infty^2]$, $\E[W_t\lvert \mathscr{F}_{t-1}'] = 0$ and $\lvert V_t \rvert \leq 1$. As a result, \begin{align*}
	\text{Var}(W_t\lvert \mathscr{F}_{t-1}')  &\leq  \E[ W_t^2 \lvert \mathscr{F}_{t-1}' ]=\text{Var}[V_t \lvert \mathscr{F}_{t-1}']\leq \E[V_t^2 \lvert \mathscr{F}_{t-1}']\\
	&\leq \E[V_t \lvert \mathscr{F}_{t-1}'] = \E_{x}[\lVert  {\bm g}_t(x,a) - {\bm g}^*(x,a)\rVert_\infty^2 ]
\end{align*}
Now by Lemma~\ref{lem-Freedman}, we have with probability at least $1-2\delta\log T$, {\footnotesize \begin{align*}
 &\big \lvert  \sum_{t = 1}^M \lVert  {\bm g}_t(x_t,a) - {\bm g}^*(x_t,a)\rVert_\infty^2 - \E_{x}[\lVert  {\bm g}_t(x,a) - {\bm g}^*(x,a)\rVert_\infty^2]  \big \rvert  = \big \lvert  \sum_{t = 1}^M W_t \big \rvert\\
 \lesssim & \sqrt{\log (1/\delta)}\sqrt{\max\{\sum_{t = 1}^M \E_{x}[\lVert  {\bm g}_t(x,a) - {\bm g}^*(x,a)\rVert_\infty^2,\log (1/\delta)\}  }
\end{align*}}
Thus when $\max\{\sum_{t = 1}^M \E_{x}[\lVert  {\bm g}_t(x,a) - {\bm g}^*(x,a)\rVert_\infty^2 > \log (1/\delta),$
we have \begin{align*}
     \sum_{t = 1}^M \E_{x}[ \lVert  {\bm g}_t(x,a) - {\bm g}^*(x,a)\rVert_\infty^2] \lesssim \Regsq^c(M) + \sqrt{\log (1/\delta)\sum_{t = 1}^M \E_{x}[ \lVert  {\bm g}_t(x,a) - {\bm g}^*(x,a)\rVert_\infty^2]},
\end{align*}
which then implies 
\begin{align*}
     \sum_{t = 1}^M \E_{x}[ \lVert  {\bm g}_t(x,a) - {\bm g}^*(x,a)\rVert_\infty^2] \lesssim \Regsq^c(M)+\log (1/\delta).
\end{align*}
That leads to the desired result with replacing $\delta$ by $\frac{\delta}{2\log T}$.
\end{proof}

\subsection{Proof of Theorem~\ref{thm-finer-regret}}

\begin{proof}[Proof of Theorem~\ref{thm-finer-regret}]
Since in Algorithm~\ref{alg-squaredCBwK-Z}, we use $O(KT_0)$ steps to take exploration, the expected regret incurred by this stage is upper bounded by $O(\dfrac{T\OPT}{B}+1)KT_0$. For the second stage, since $Z \leq CT\OPT/B$, the expected regret is upper bounded by \begin{align*}
    (\frac{T\OPT}{B}+1) \sqrt{KT \big( \Regsq^r(T) + \Regsq^c(T)+\log T \big)}
\end{align*}
Then Theorem~\ref{thm-finer-regret} holds by add the regret in these two stages together.
\end{proof}

\subsection{Proof of Lemma~\ref{lem-opt-error}}
\begin{proof}[Proof of Lemma~\ref{lem-opt-error}]
Similar to \cite{agrawal2016linear} in linear case, our proof relies on the result about the ``intermediate sample optimal'' $\overline{\OPT}^\epsilon$, which is the value of \begin{equation}\label{eq-intermediate-LP}
	\begin{aligned}
 	  \max_{\bm p \in (\Delta^K)^{T_0} } & \dfrac{1}{T_0} \sum_{t=KT_0+1}^{(K+1)T_0}  f^*(x_t,a) p_a(x_t)\\
 	  \text{subject to } & \dfrac{1}{T_0}\sum_{t = KT_0+1}^{(K+1)T_0} \bm g^*(x_t,a) p_a(x_t) \leq \bm B/T + \epsilon \bm 1.  
\end{aligned}
\end{equation}

Applying the Lemma~F.4 and F.6 of \cite{agrawal2014fast} in the same way as \cite{agrawal2016linear}\footnote{Notice that our definition of $\epsilon$ and $\OPT$ corresponds to the $\gamma/T$ and $\OPT /T$  in \cite{agrawal2016linear}. }  leads to \begin{align}\label{eq-inter-bound}
	\OPT - \epsilon \leq  \overline{\OPT}^\epsilon \leq \OPT + 2 \epsilon( \frac{T\OPT}{B}+ 1 )
\end{align}
with probability $1-O(1/T^2)$ when $\epsilon \geq 2\sqrt{\frac{\log (T_0d)}{T_0}} .$

Now it sufficient to bound $\widehat{\OPT} (T_0)$ and $\overline{\OPT}^{\epsilon}:$ 
Firstly noticing that since we have used $ T_0 $ i.i.d. samples from  $P_X$  to estimate $f(\cdot,a),\bm g(\cdot,a)$, we have with probability at least $1-O(1/T^2),$ \begin{align}\label{eq-high-prob-estimation}
	\E_{x} \big [ \big(\hat{f}_0(x,a)  - f(x,a)\big)^2 \big] \leq \mathcal{E}_{\delta_0}(T_0;\mF),\quad \E_{x} \big [ \big\lVert \hat{\bm g}_0(x,a)  - \bm g(x,a)\big\rVert_\infty^2 \big] \leq d\mathcal{E}_{\delta_0}(T_0;\mG), \quad \forall a\in [K],    
\end{align}
where $\delta_0 = 1/{TK}^2. $
Thus we have with probability at least $1-O(1/T^2),$ for all $\bm p \in (\Delta^{K})^{T_0} $, \begin{align*}
	  \big \lvert \sum_{t = KT_{0}+1}^{(K+1)T_0} \sum_{a}p_{a}(x_t) [ \hat{f}_0(x,a) - f^*(x,a)]  \big\rvert &\leq \sum_{t = KT_{0}+1}^{(K+1)T_0} \max_a \lvert \hat{f}_0(x,a) - f^*(x,a) \rvert   \\
	 &\leq \sum_{t = KT_0+1}^{(K+1)T_0}  \sqrt{\E[\max_{a}\lvert \hat{f}_0(x,a) - f^*(x,a) \rvert^2 ]}+ 4\sqrt{ T_0\log T }, \\
  &\leq T_0 \sqrt{K\mathcal{E}_{\delta_0}(T_0;\mathcal{F})}+4 \sqrt{T_0\log T}
\end{align*}
where the second line is by Cauchy-Schwartz inequality and  Hoeffding's inequality and the third line is by \eqref{eq-high-prob-estimation}. Applying the same argument to $\hat{\bm g}_0,$ we get 
\begin{align*}
	\Probability \bigg(\big\lvert    \sum_{t = KT_0+1}^{(K+1)T_0} \sum_{a}p_{a}(x_t) [ \hat{f}_0(x,a) - f^*(x,a)]\big\rvert \geq  \underbrace{T_0 \sqrt{ K\mathcal{E}_{\delta_0}(T_0;\mF)  }+ 4\sqrt{ T_0\log T }}_{\leq T_0  \mathcal{R}(T_0)},  \forall \bm p \in (\Delta^K)^{T_0}  \bigg)  \lesssim  1/T^2. 
\end{align*}and
$$
	\Probability \bigg(\big\lVert    \sum_{t = KT_0+1}^{(K+1)T_0} \sum_{a}p_{a}(x_t) [ \hat{\bm g}_0(x,a) - \bm g^*(x,a)]\big\rVert_\infty \geq  \underbrace{\big (T_0 \sqrt{K \mathcal{E}_{\delta_0}(T_0;\mG)  }+  4\sqrt{ T_0\log T }\big )}_{\leq T_0 \mathcal{R}(T_0)},  \forall \bm p \in (\Delta^K)^{T_0}   \bigg)  \lesssim  1/T^2. $$

Denoting $\bar {\bm p}$ the optimal solution of \eqref{eq-intermediate-LP} for  $2\sqrt{\frac{\log(Td)}{T_0}} \leq \epsilon \leq \mathcal{R}(T_0) $, we have 
 the deviation bound of $\hat{\bm g}_0$ implies $\bar{\bm p}$ is feasible in \eqref{eq-Z-est-LP} with probability $1-O(1/T^2)$, thus combining the deviation bound of $\hat{f}_0,$ we have with probability at least $1-O(1/T^2)$, {\small \begin{align}\label{eq-lb-Z-LP}
	\widehat{\OPT} (T_0) \geq  \dfrac{1}{T_0} \sum_{t = KT_0+1}^{(K+1)T_0} \sum_a \bar{p}_a(x_t)\hat{f}_0(x_t,a) \geq  \dfrac{1}{T_0} \sum_{t = KT_0+1}^{(K+1)T_0}\sum_a \bar{p}_a(x_t){f}^*(x_t,a)-\mathcal{R}(T_0)  = \overline{\OPT}^{\epsilon} -\mathcal{R}(T_0).
\end{align}}
On the other hand,  for $\hat{\bm p}$ the optimal solution of \eqref{eq-Z-est-LP}, we have with high probability $\hat{\bm p}$ is feasible for \eqref{eq-intermediate-LP} with $\epsilon_0 = 2 \mathcal{R}(T_0),$ thus with probability at least $1-O(1/T^2),$ we have{\small \begin{align}\label{eq-ub-Z-LP}
	\widehat{\OPT}(T_0) = \dfrac{1}{T_0} \sum_{t = KT_0+1}^{(K+1)T_0} \sum_{a}\hat{p}_a(x_t)\hat{f_0}(x_t,a)\leq 
	\dfrac{1}{T_0} \sum_{t = KT_0+1}^{(K+1)T_0} \sum_{a}\hat{p}_a(x_t){f_0^*}(x_t,a)+ \mathcal{R}(T_0) \leq \overline{\OPT}^{\epsilon_0} + \mathcal{R}(T_0).
\end{align}}
Now combine \eqref{eq-inter-bound},\eqref{eq-lb-Z-LP},\eqref{eq-ub-Z-LP}, we get with probability at least $1-O(1/T^2),$ \begin{align*}
	\OPT  \leq \widehat{\OPT}(T_0)+ \mathcal{R}(T_0) \leq \OPT +6 \mathcal{R}(T_0)(\dfrac{T\OPT}{B}+{1}),
\end{align*}
as desired.
\end{proof}

\subsection{Result in Linear CBwK}
In linear CBwK setting, suppose both $\mathcal{F}$ and $\tilde{\mathcal{G}}$ are $m$-dimensional linear classes. If we select the online oracles $\mathcal{R}^r,\tilde{\mathcal{R}}^c$ as in Table~\ref{tbl-table-comparison} and offline oracles $\mathcal{R}^r_{est}, \mathcal{R}^c_{est}$ as the OTB oracle constructed from $\mathcal{R}^r,\mathcal{R}^c$, we have then \begin{align}\label{eq-example-linear-general-B}
    \Regsq^r(T)\lesssim m \log T, \Regsq^c(T)\lesssim md \log T, \mathcal{E}_{T_0}(\mathcal{F}) \lesssim \dfrac{m\log T}{T_0},\mathcal{E}_{T_0}(\mathcal{G}) \lesssim \dfrac{md\log T}{T_0}. 
\end{align}
Now bringing \eqref{eq-example-linear-general-B} to Theorem~\ref{thm-finer-regret} and selecting $T_0 =  (md)^{1/3} \sqrt{T/K} $leads to 
\begin{align*}
    \Reg(T) \lesssim (\frac{T\OPT }{B}+1 ) ( \sqrt{KTdm})
\end{align*}
when $B = \tilde{\Omega}\big((md)^{1/3}\sqrt{KT}\big).$

\subsection{Results in Nonparametric CBwK}
In nonparametric CBwK setting, suppose both $\mathcal{F}$ and $\tilde{\mathcal{G}}$ are $p$-nonparametric classes. If we select the online oracles $\mathcal{R}^r,\tilde{\mathcal{R}}^c$ as in Table~\ref{tbl-table-comparison} and offline oracles $\mathcal{R}^r_{est}, \mathcal{R}^c_{est}$ as the OTB oracle constructed from $\mathcal{R}^r,\mathcal{R}^c$, we have then \begin{align}\label{eq-example-nonpara-general-B}
    \Regsq^r(T)\lesssim  (KT)^{1-\frac{2}{2+p}}, \Regsq^c(T)\lesssim  d (KT)^{1-\frac{2}{2+p}}, \mathcal{E}_{T_0}(\mathcal{F}) \lesssim \dfrac{(KT_0)^{1-\frac{2}{2+p}}}{T_0},\mathcal{E}_{T_0}(\mathcal{G}) \lesssim \dfrac{d(KT)^{1-\frac{2}{2+p}}}{T_0}. 
\end{align}
Now bringing \eqref{eq-example-nonpara-general-B} to Theorem~\ref{thm-finer-regret} and selecting $T_0 =  d^{\frac{2+p}{6+2p}} K^{\frac{-1}{2+p}}  T^{\frac{1+p}{2+p}}$ leads to 
\begin{align*}
    \Reg(T) \lesssim \tilde{O}\big((\frac{T\OPT}{B}+1)\sqrt{d}(KT)^{\frac{1+p}{2+p}}  \big)  
\end{align*}
when
$B = \tilde{\Omega}( d^{\frac{2+p}{6+2p}} (KT)^{\frac{3+p}{4+2p}}).$

\section{NUMERICAL RESULTS}\label{simulations}

We provide simulation results in this section to validate the time horizon ($T$), dimension ($m$), and number of arms ($K$) dependencies of SquareCBwK, which employs Newtonized GLMtron and Online Gradient Descent oracles, for the linear CBwK. We also conduct a performance comparison of SquareCBwK with LinUCB \citep{agrawal2016linear}.

For general distributions of $x_{t,a}$, solving the OPT value from population linear programming \eqref{eq-OPT-static} is challenging which renders the computation of regret, $T\text{OPT} - \E[\sum_{t = 1}^\tau r_{t,a_t} ]$, an intractable task. Therefore, in our subsequent simulations, we adhere to the fixed context setting, which simplifies the process of determining the OPT value to that of solving a linear programming, as outlined in Section~\ref{sec-exp-setting}. In fact, the fixed context setting is a special case of i.i.d. context by letting the distribution of contexts be the point mass distribution.

\subsection{Experiment Setting}\label{sec-exp-setting}
Throughout the experiment, we assume the linear structured reward and cost classes. For any fixed dimension $m$, arm number $K \leq m-1,$ and number of constraints $4\leq d\leq m-1$, we set
\begin{enumerate}
    \item \textbf{Underlying parameters:}$$\theta_0 = \frac{1}{\sqrt{2}}(e_1+e_2),\theta_1 = \frac{1}{\sqrt{2}}(e_1+e_3), \theta_2 = \frac{1}{2}(e_2+e_3+e_4+e_5),\theta_i = e_{i+1}, 3\leq i\leq d. $$
    \item \textbf{Fixed context set:} At every round $t$, the contexts $x_{t,a}$ are given by $ x_{t,a} =  \frac{1}{\sqrt{2}} e_1+e_{a+1}, a \in [K] $
    \item \textbf{Generation of rewards and costs:} At every round $t,$ after an action $a_t\in [K]$ is selected 
    \begin{align*}
        r_t = \langle x_{t,a_t},\theta_0 \rangle + \epsilon_{t,0}, \quad
        c_{t,i} = \langle x_{t,a_t},\theta_i \rangle + \epsilon_{t,i}, i \in [d] 
    \end{align*}  
    with $\epsilon_{t,i}\sim_{i.i.d.} \mathcal{N}(0,0.2).$ 
\end{enumerate}

In the experiments that follow, we will simulate the three algorithms in Table~3 independently with one varying hyper-parameter, selected from $m,K$, and $T$, while keeping the remaining two hyper-parameters constant. This will allow us to validate the theoretical dependency as detailed in Table~3. The source code for reproducing these results can be found in \url{https://github.com/quejialin/SquareCBwK}

{\centering
\resizebox{0.5\linewidth}{!}{
	  \begin{tabular}{|c|c|c|c|}
	  \hline
	 Parameter & GLMtronNewton & Online GD & LinUCB \\
  	  \hline
	 $K$ &  $O(\sqrt{K})$ & $O(\sqrt{K})$ &  $O(1)$ \\
  	  \hline
	 $m$ & $O(\sqrt{m})$ & $O(1)$ & $O({m})$ \\  
            \hline
	 $T$ & $O(\sqrt{T})$ & $O(T^{3/4})$ & $O(\sqrt{T})$ \\
  \hline
	  \end{tabular}}
    \captionof{table}{Theoretical dependency of SquareCBwK with Newtonized GLMtron oracle, SquareCBwK with Online Gradient Descent oracles, and 
 LinUCB algorithms \citep{agrawal2016linear}  on $K,m,T$}

}

\begin{figure}[htb]
    \centering
\subcaptionbox{\label{1}}{\includegraphics[width = .43\linewidth]{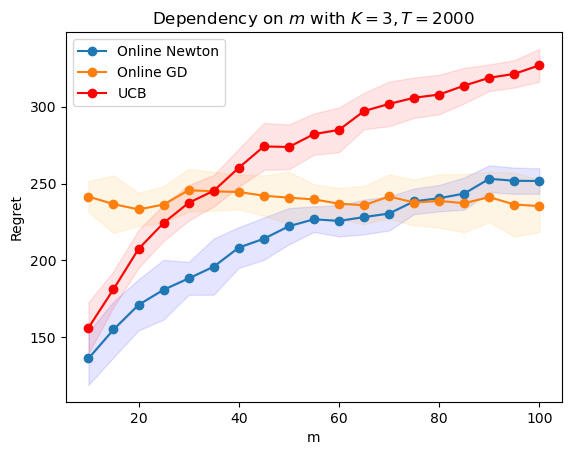}}\hfill
\subcaptionbox{\label{2}}{\includegraphics[width = .43\linewidth]{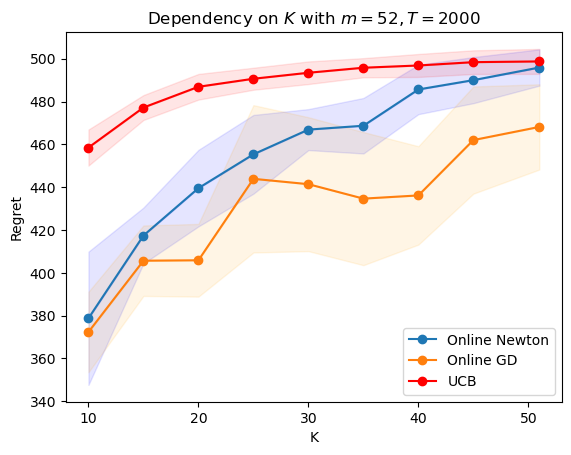}}

\subcaptionbox{\label{1}}{\includegraphics[width = .43\linewidth]{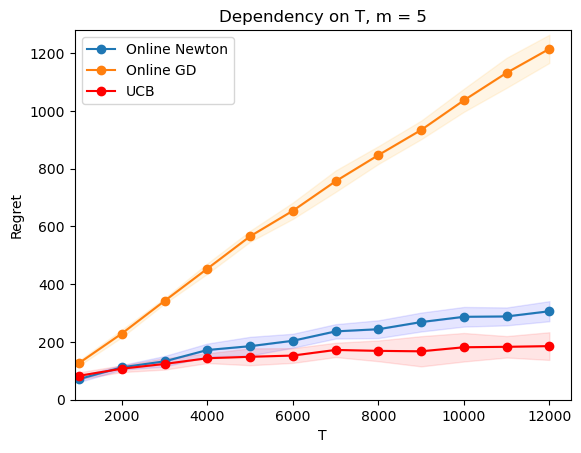}}\hfill
\subcaptionbox{\label{2}}{\includegraphics[width = .43\linewidth]{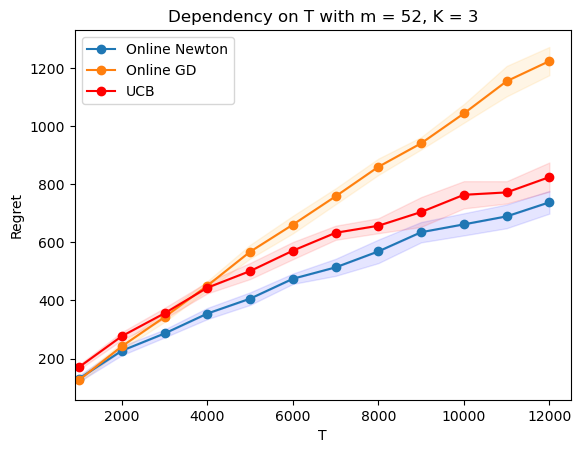}}

\caption{Dependency of different algorithms on different parameters $m,K,T$, where the blue curve corresponds to the regret of SquareCBwK with Newtonized GLMtron oracle, the orange curve shows regret performance of SquareCBwK with online Gradient oracle, and the red curve represents the regret of LinUCB.}
    \label{fig:my_label}
\end{figure}

\subsection*{Dependency on $m$}

To examine the dependency on $m$, we fix $K = 3, T = 2000$ and simulate different algorithms with $m\in \{10,14,19,\dots,96,101\}$ over 10 times. The regret curve is presented in Figure~1(a), from which we can see that the regret of the SquareCBwK with Online-GD oracle is almost unaffected by $m$, and that the regret of GLMtron oracle grows at a slower pace than LinUCB as $m$ increases, matching the theoretical guarantees in Table~3.

\subsection*{Dependency on $K$}

To examine the dependency on $K$, we fix $m = 52, T = 2000$ and simulate different algorithms with $K\in \{5,10,\dots,45,50\}$ over 10 times. Note that here we pick a large $m=52$ since we need to make sure the condition $K<=m-1$ holds. The regret curve is presented in Figure~1(b). Although the dependency of LinUCB on $K$ is better than SquareCBwK with GLMtron or Online-GD oracles, the regret of LinUCB is larger in the large-$m$ regime.

\subsection*{Dependency on $T$}

To examine the dependency on $T$, we fix $K = 3$ and simulate different algorithms with $T\in \{1000,1100,\dots,12000\}$. Here we also choose $m \in \{5,52\}$ to compare the influence of large and small dimension respectively. The regret curves are presented in Figure~1(c) and Figure~1(d).
It can be shown that the regret of SquareCBwK with Online-GD grows much faster than SquareCBwK with GLMtron and LinUCB for both $m$ as $T$ increases. The regret of GLMtron and LinUCB are of the same order. Although LinUCB outperforms SquareCBwK with GLMtron slightly in the small $m$ regime, SquareCBwK with GLMtron exhibits an improved performance in the large $m$ regime, which again verifies the dependency on $T$ and $m$ in Table~3 of different algorithms.

From the above comparisons, we can see that in the large dimension setting of linear CBwK, SquareCBwK with online-GD oracle is independent of $m$ although with worse dependency on $T$, and SquareCBwK with Newtonized GLMtron oracles exhibits a better dependency on dimension $m$ than LinUCB, while keeping the same order of $T$ as LinUCB. These results together verify the superiority of our algorithm in the large dimension setting of linear CBwK.

\end{document}